\documentclass{article}

\usepackage[numbers]{natbib}
\usepackage{arxiv}

\usepackage[utf8]{inputenc} 
\usepackage[T1]{fontenc}    
\usepackage{hyperref}       
\usepackage{url}            
\usepackage{booktabs}       
\usepackage{amsfonts}       
\usepackage{nicefrac}       
\usepackage{microtype}      

\usepackage{amsthm}

\usepackage{subcaption}
\usepackage{mwe}

\usepackage{amsfonts} 
\usepackage{mathtools} 
\usepackage{bm} 

\usepackage{tikz}
\usetikzlibrary{bayesnet}

\usepackage{booktabs}
\usepackage{multirow}
\usepackage{siunitx}
\usepackage{wrapfig}

\usepackage{algorithm}
\usepackage{algorithmic}

\newtheorem{post}{Postulate}
\newtheorem{assp}{Assumption}
\newtheorem{dfnt}{Definition}

\newtheorem{prop}{Proposition}
\newtheorem{lema}{Lemma}
\DeclareMathOperator{\tr}{tr}

\title{A Kernel Embedding-based Approach for Nonstationary Causal Model Inference}

\author{
    Shoubo Hu{\small $^{\ast}$}, Zhitang Chen{\small $^{\dagger}$}, Laiwan Chan{$^{\ast}$} \\
  $^{\ast}$The Chinese University of Hong Kong; $^{\dagger}$Huawei Noah's Ark Lab \\
  $^{\ast}$\texttt{\{sbhu, lwchan\}@cse.cuhk.edu.hk}; $^{\dagger}$\texttt{chenzhitang2@huawei.com }
}

\begin{document}
\maketitle

\begin{abstract}
Although nonstationary data are more common in the real world, most existing causal discovery methods do not take nonstationarity into consideration. In this letter, we propose a kernel embedding-based approach, ENCI, for nonstationary causal model inference where data are collected from multiple domains with varying distributions. In ENCI, we transform the complicated relation of a cause-effect pair into a linear model of variables of which observations correspond to the kernel embeddings of the cause-and-effect distributions in different domains. In this way, we are able to estimate the causal direction by exploiting the causal asymmetry of the transformed linear model. Furthermore, we extend ENCI to causal graph discovery for multiple variables by transforming the relations among them into a linear nongaussian acyclic model. We show that by exploiting the nonstationarity of distributions, both cause-effect pairs and two kinds of causal graphs are identifiable under mild conditions. Experiments on synthetic and real-world data are conducted to justify the efficacy of ENCI over major existing methods.
\end{abstract}


\section{Introduction}
Causal inference has been given rise to extensive attention and applied in several areas including statistics, neuroscience and sociology in recent years. An efficient approach for causal discovery is to conduct randomized controlled experiments. These experiments, however, are usually very expensive and sometimes practically infeasible. Therefore, causal inference methods using passive observational data take center stage, and many of them have been proposed, especially in the past ten years.

Existing causal inference methods that use passive observational data can be roughly categorized into two classes according to their objectives. One class of methods aim at identifying the variable that is the cause of the other in a variable pair \citep{hoyer2009nonlinear,zhang2009identifiability,janzing2012information, chen2014causal}, which is often termed a cause-effect pair. Most of the methods in this class first model the relation between the cause and the effect using a functional model with certain assumptions. Then they derive a certain property which only holds in the causal direction and is violated in the anticausal direction to infer the true causal direction. This kind of widely used property is often termed cause-effect asymmetry. For example, the additive noise model (ANM) \citep{hoyer2009nonlinear} represents the effect as a function of the cause with an additive independent noise: $Y = f(X) + E_{Y}$. The authors showed that there is no model of the form $X = g(Y) + E_{X}$ that admits an ANM in the anticausal direction for most combinations $\left(f, p(X), p(E_{Y})\right)$. Therefore, the inference of ANM is done by finding the direction that fits ANM better. Similar methods include postnonlinear model (PNL) \citep{zhang2009identifiability} and information geometric causal inference (IGCI) \citep{janzing2012information}. Recently, a kernel-based, EMD (or abbreviation for EMbeDding) \citep{chen2014causal} using the framework of IGCI is proposed. EMD differs from the previous methods in the sense that it does not assume any specific functional model, but it still resorts to find the cause-effect asymmetry.

The other class of methods aims at recovering the structure of causal graphs. Constraint-based methods \citep{spirtes1991algorithm, spirtes1991probability, perlcausality, spirtes2000causation, Cheng200243}, which belong to this class, exploit the causal Markov condition and have been widely used in the social sciences, medical science, and bioinformatics. However, these methods allow one only to obtain the Markov equivalent class of the graph and are of high computational cost. In 2006, a linear nongaussian acyclic model (LiNGAM) \citep{shimizu-et-al:lingam} which exploits the nongaussian property of the noise, was showed to be able to recover the full causal structure by using independent component analysis (ICA) \citep{comon1994independent, hyvarinen2000independent}. To avoid the problem that ICA may result in a solution of local optima, different methods \citep{shimizu2011directlingam, hyvarinen2013pairwise} were proposed to guarantee the correctness of the causal order of variables in the causal graph.

Both classes of existing methods are based on the assumption that all observations are sampled from a fixed causal model. By ``fixed causal model,'' we mean that the (joint) distribution of variables and the mechanism mapping cause(s) to effect(s) are unchanged during the data collecting process. For example in an ANM $Y = f(X) + E_{Y}$, both the distribution of the cause $p(X)$ and the causal mechanism $f$ are assumed to be fixed. Although some of these methods do achieve inspiring results and provide valuable insights for subsequent research, data generated from a varying causal model are much more common in practice and existing methods based on a fixed causal model would come across some problems when applied to varying causal models \citep{zhang2015discovery}. Therefore, we consider causal models where distributions of variables and causal mechanisms vary across domains or over different time periods and call these models \emph{non-stationary causal models}. An example is the model of daily returns of different stocks. The distribution of the return of each stock varies with the financial status, and the causal mechanisms between different stocks also vary according to the relations between these companies. Recently, a method called Enhanced Constraint-based Procedure (ECBP) was proposed for causal inference of non-stationary causal models \citep{zhang2015discovery}. The authors resorted to an index variable $C$ to quantify the nonstationarity and proposed ECBP, which is built on constraint-based methods to recover the skeleton of the augmented graph, which consists of both observed variables $\mathbf{V}$ and some unobserved quantities determined by $C$. They also showed that it is possible to infer the parent nodes of variables adjacent to $C$ (termed \emph{$C$-specific variables}) and proposed a measure to infer the causal direction between each $C$-specific variable and its parents. However, their method fails to ensure the recovery of the full causal structure, which is due to the limitation of methods that rely on conditional independence test. In contrast, our method, which is proposed originally for cause-effect pairs inference, is also extended to infer the complete causal structure of two kinds of graphs by transforming the nonstationarity into a LiNGAM model.

In this paper, we introduce a nonstationary causal model and develop algorithms, which we call \emph{embedding-based nonstationary causal model inference} (ENCI) for inferring the complete causal relations of the model. Our model assumes that the underlying causal relations (i.e. the causal direction of a cause-effect pair or the causal structure of a graph) are fixed, whereas the distributions of variables and the causal mechanisms (i.e. the conditional distribution of the effect given the cause(s)) change across domains or over different time periods. To infer the nonstationary causal model, ENCI reformulates the relation among variables into a linear model in the Reproducing Kernel Hilbert Space (RKHS) and leverages the identifiability of the linear causal model to tackle the original complicated problem. Specifically, for a cause-effect pair, we embed the variation of the density of each variable into an RKHS to transform the original unknown causal model to a linear nongaussian additive model \citep{kano2003causal} based on the independence between the mechanism generating the cause and the mechanism mapping the cause to the effect. Then we infer the causal direction by exploiting the causal asymmetry of the obtained linear model. We also extend our approach to discover the complete causal structure of two kinds of causal graphs in which the distribution of each variable and the causal mechanism mapping cause(s) to effect(s) vary and the causal mechanism could be nonlinear. 

This paper is organized as follows. In section 2, we formally define our model and objective of causal inference. In section 3, some preliminary knowledge of reproducing kernel Hilbert space embedding is introduced. In section 4, we elaborate our methods for cause-effect pairs. In section 5, we extend our methods to two kinds of causal graphs. In section 6, we report experimental results on both synthetic and real-world data to show the advantage of our approach over existing ones.

\section{Problem Description}
In this section we formalize the nonstationary causal model and the objective of our causal inference task. For a pair of variable $X$ and $Y$, we consider the case where $X$ is the cause and $Y$ is the effect without loss of generality throughout this paper.

\subsection{Non-stationary Causal Model}
We assume the data generating process of a cause-effect pair fulfills the following properties: 
\begin{itemize}
	\item The causal direction between $X$ and $Y$ stays the same throughout the process.
	\item Observations are collected from $N$ different domains. The density of the cause $\left(p(X)\right)$ and the conditional density of the effect given the cause $\left(p(Y|X)\right)$ are fixed within each domain.
	\item $p(X)$ and $p(Y|X)$ vary in different domains.
\end{itemize}
We call this a \emph{nonstationary causal model} due to the variation in distributions over domains. The data-generating process is illustrated in Figure~\ref{fig:model}.
\begin{figure}[hbtp]
	\label{model}
	\begin{center}
		\includegraphics[width = 0.8\linewidth]{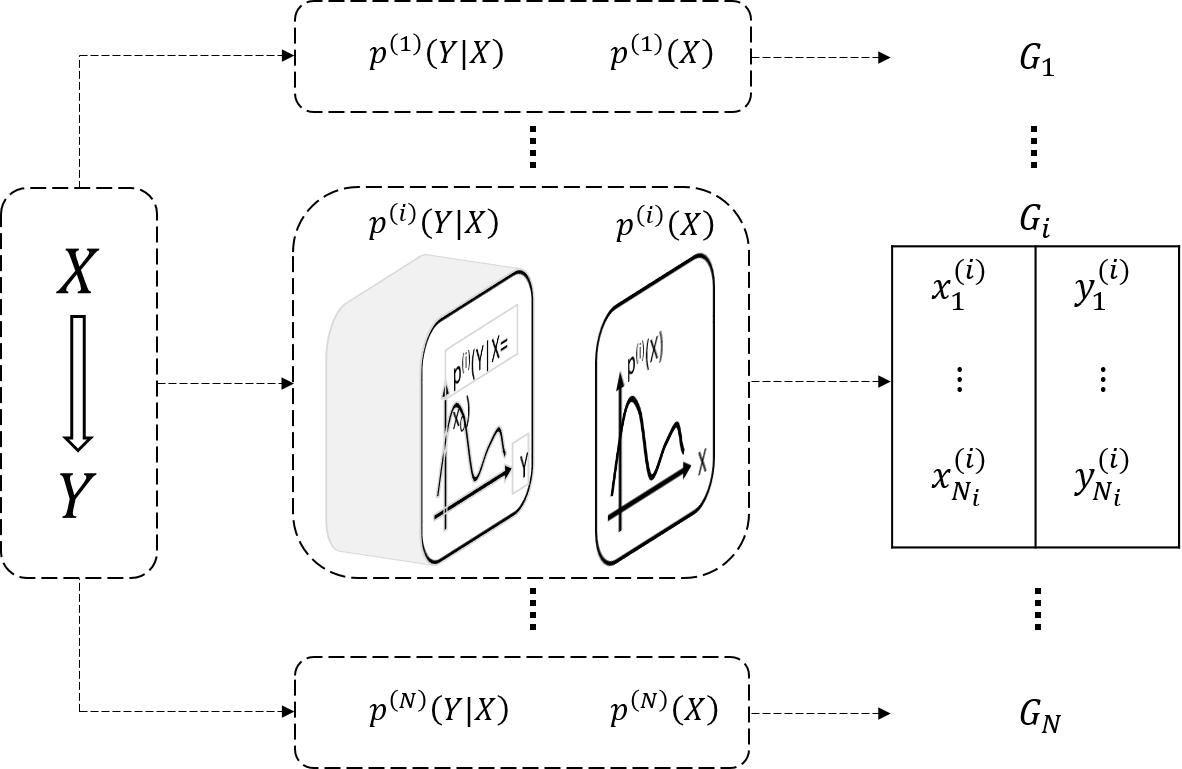}
	\end{center}
	\caption{Data generating process of non-stationary causal model}
	\label{fig:model}
\end{figure}

The collection of data obtained from each domain is called a data group $G_{i}$, and the entire data set is denoted by $\mathbf{G} = \lbrace G_{1}, G_{2}, \dots, G_{N} \rbrace$. This nonstationarity over groups is common in the real world, as the observations we obtained are usually collected over different time periods or from different sources (e.g. different geographical regions or experimental settings).

\subsection{Objective of Non-stationary Causal Model Inference}
Our goal of nonstationary causal model inference is, by exploiting the variation of distributions in different groups, to accurately estimate the causal direction between $X$ and $Y$. We also extend, our approach to learn the full causal structure of two kinds of causal graphs by transforming their relationship among groups into a LiNGAM model. For clarity, we list some of the notations we use in the following sections in Table~\ref{notation}.
\begin{table}[!ht]
	\caption{Notations}
	\label{notation}
	\begin{center}
		\begin{tabular}{ll}
			\toprule
			Symbol & Description \\
			\midrule
			$p^{(i)}(X), p^{(i)}(Y)$ & Density of $X$, $Y$ in group $i$ \\ 
			$\overline{p}(X)$ & Base of the density of $X$ \\ 
			$\Delta p^{(i)}(X)$ & Variation of the density of $X$ in group $i$ \\
			$p^{(i)}(Y|X)$ & Conditional density of $Y$ given $X$ in group $i$ \\
			$\overline{p}(Y|X)$ & Base of the conditional density of $Y$ given $X$ \\
			$\Delta p^{(i)}(Y|X)$ & Variation of the conditional density of $Y$ given $X$ in group $i$ \\
			$\mathcal{X}, \mathcal{Y}$ & domain of variable $X$, $Y$ \\
			$\mu^{(i)}_{\otimes X}$, $\mu^{(i)}_{\otimes Y}$ & Mean embedding of $p^{(i)}(X)$ in $\mathcal{X} \otimes \mathcal{X}$, $p^{(i)}(Y)$ in $\mathcal{Y} \otimes \mathcal{Y}$ \\
			$\overline{\mu}_{\otimes X}$, $\overline{\mu}_{\otimes Y}$ & Mean embedding of $\overline{p}(X)$ in $\mathcal{X} \otimes \mathcal{X}$, $\overline{p}(Y)$ in $\mathcal{Y} \otimes \mathcal{Y}$ \\
			$\Delta \mu^{(i)}_{\otimes X}$, $\Delta \mu^{(i)}_{\otimes Y}$ & Mean embedding of $\Delta p^{(i)}(X)$ in $\mathcal{X} \otimes \mathcal{X}$, $\Delta p^{(i)}(Y)$ in $\mathcal{Y} \otimes \mathcal{Y}$ \\
			\bottomrule
		\end{tabular}
	\end{center}
\end{table}

\section{Hilbert Space Embedding of Distributions}
Kernel embedding-based approaches represent probability distributions by elements in a reproducing kernel Hilbert space (RKHS) and it serves as the main tool in this letter to characterize distributions.

An RKHS $\mathcal{F}$ over $\mathcal{X}$ with a kernel $k$ is a Hilbert space of functions $f:\mathcal{X} \rightarrow \mathbb{R}$. Denoting its inner product by $\langle \cdot, \cdot \rangle_{\mathcal{F}}$, RKHS $\mathcal{F}$ fulfills the reproducing property $\langle f(\cdot), k(x,\cdot) \rangle_{\mathcal{F}} = f(x)$. People often regard $\phi(x) \coloneqq k(x,\cdot)$ as a feature map of $x$. Kernel embedding of a marginal density $p(X)$ \citep{smola2007hilbert} is defined as the expectation of its feature map:
\begin{equation}
\mu_{X} \coloneqq \mathbb{E}_{X}[\phi(X)] = \int _{ \mathcal{X}  }^{  }{ \phi(x)p(x)dx },
\label{knl:1}
\end{equation}
where $\mathbb{E}_{X}[\phi(X)]$ is the expectation of $\phi(X)$. It has been shown that $\mu_{X}$ is guaranteed to be an element in RKHS if $\mathbb{E}_{X}[k(X,X)]<\infty$ is satisfied.  It  is also generalized to joint distribution using tensor product feature spaces. The kernel embedding of a joint density $p(X, Y)$ is defined as 
\begin{align}
\mathcal{C}_{XY} & \coloneqq \mathbb{E}_{XY}[\phi(X) \otimes \phi(Y)] = \int_{\mathcal{X} \times \mathcal{Y}}{ \phi(x) \otimes \phi(y)p(x, y)dxdy}.
\label{knl:2}
\end{align}
Similarly, we have that $\mathcal{C}_{XX} \coloneqq \mathbb{E}_{X}[\phi(X) \otimes \phi(X)]$. The embedding of conditional densities is viewed as an operator that maps from $\mathcal{F}$ to $\mathcal{G}$ which is an RKHS over $\mathcal{Y}$ \citep{song2009hilbert}. Imposing that the conditional embedding satisfies the following two properties:
\begin{align}
& \mu_{Y|x} \coloneqq \mathbb{E}_{Y|x}[\phi(Y)|x] = \mathcal{U}_{Y|X}k(x, \cdot), \label{req:1}\\
& \mathbb{E}_{Y|x}[g(Y)|x]=\langle g, \mu_{Y|x} \rangle_{\mathcal{G}}, \label{req:2}
\end{align}
where $g \in \mathcal{G}$ and $\mu_{Y|x}$ is kernel embedding of marginal density $p(Y|X=x)$, \cite{song2009hilbert} showed that conditional embedding can be defined as $\mathcal{U}_{Y|X} \coloneqq \mathcal{C}_{YX} \mathcal{C}_{XX}^{-1}$ to fulfill equations \ref{req:1} and \ref{req:2}. In the following sections, we follow the definition of kernel mean embedding and embed distributions in a tensor product space to represent distribution of each group.

\section{Embedding-based Nonstationary Causal Model Inference}
In this section we introduce our proposed approach to infer the causal structure of nonstationary causal models.
\subsection{Basic Idea}
Currently, the most widely used idea of inferring causal direction is to quantify the independence between the mechanism generating the cause and the mechanism mapping the cause to the effect. One way to interpret the independence between these two mechanisms is to measure the independence between the cause and the noise. ANM and PNL lie in this field and LiNGAM methods could also be interpreted from this viewpoint \citep{hyvarinen2013pairwise}. We adopt a different interpretation which uses the independence between the marginal distribution of the cause and the conditional distribution of the effect given the cause to capture the independence between these two mechanisms and further exploit causal asymmetry. This kind of independence has also been used in many existing causal inference methods \citep{janzing2009telling, janzing2012information, chen2014causal}. We formalize this independence in postulate 1:
\begin{post}\label{post1}
	The mechanism generating the cause and the mechanism mapping the cause to the effect are two independent natural processes.
\end{post}
\cite{janzing2012information} proposed this postulate and developed information geometry causal inference (IGCI). IGCI uses the density of the cause to characterize the first mechanism and the derivative of the function mapping the cause to the effect to characterize the second. In our approach, the variation of the marginal density of the cause is used to characterize the first mechanism, which is similar to IGCI. What differs from IGCI is that we use the variation of the conditional density to characterize the second mechanism. In subsequent sections, we introduce how we obtain the variation of densities and how we infer the causal direction based on the independence between them. 

\subsection{Decomposition of Distributions}
Given the entire data set $\mathbf{G} = \{G_{1}, G_{2}, \dots, G_{N}\}$, which consists of $N$ groups, we make use of the variation of densities in each group. To obtain the variation, we first compute the mean of marginal densities and conditional densities of all groups as
\begin{align}
\overline{p}(X) = \frac{1}{N}\sum_{i=1}^{N} p^{(i)}(X), \quad \quad \overline{p}(Y|X) = \frac{1}{N}\sum_{i=1}^{N} p^{(i)}(Y|X), \label{1_1} 
\end{align}
where $p^{(i)}(X)$ is the density of $X$ and $p^{(i)}(Y|X)$ is the conditional density of $Y$ given $X$ in group $i$. We call $\overline{p}(X)$ and $\overline{p}(Y|X)$ the base of marginal and conditional densities, respectively. Then the variation of density of each group is given by:

\begin{dfnt}[Variation of density]\label{def1}
	For any $G_{i} \in \mathbf{G}$, we decompose $p^{(i)}(X)$ and $p^{(i)}(Y|X)$ into two parts: one is the base of the (conditional) density and the other is a varying part, i.e.  $p^{(i)}(X) = \overline{p}(X) + \Delta p^{(i)}(X)$ and $p^{(i)}(Y|X) = \overline{p}(Y|X) + \Delta p^{(i)}(Y|X)$. 
	We call $\Delta p^{(i)}(X)$ and $\Delta p^{(i)}(Y|X)$ the variation of the marginal and conditional density of group $i$, respectively.
\end{dfnt}

Since the base of densities is the mean of densities of all groups, $\Delta p^{(i)}(X)$ and $\Delta p^{(i)}(Y|X)$ fulfill the following properties. 
\begin{align}
\frac{1}{N}\sum_{i=1}^{N}\Delta p^{(i)}(X) \equiv 0, \quad \frac{1}{N}\sum_{i=1}^{N}\Delta p^{(i)}(Y|X) \equiv 0. \label{1_2}
\end{align}

Making use of the decomposition of distributions defined in definition \ref{def1}, we are able to analyze densities of each group with some components fixed, which finally guides us to a fixed linear causal model. We take group $i$ as an example to provides some insights before elaborating the derivations. The marginal density of effect $Y$ is given by
\begin{align}
p^{(i)}(Y)=\int{ \left(\overline{p}(x) + \Delta p^{(i)}(x) \right) \left(\overline{p}(Y|x) + \Delta p^{(i)}(Y|x)\right) dx }, \label{1_3}
\end{align}
where $\overline{p}(X)$ and $\overline{p}(Y|X)$ are the same in all groups. Therefore, we would obtain a fixed term $\int{ \overline{p}(x) \overline{p}(Y|x) dx }$ which does not change over $i$ in the expansion of equation \ref{1_3}. Although $\int{ \Delta p^{(i)}(x) \overline{p}(Y|x) dx }$ and $\int{ \overline{p}(x) \Delta p^{(i)}(Y|x) dx }$ vary over groups, they also consist of $\overline{p}(X)$ and $\overline{p}(Y|X)$ which allows us to use the invariant to formulate the relation between them into a fixed causal model. In subsequent sections, we adopt kernel embedding to transform these kinds of invariant into a linear model to infer the causal direction. 

\subsection{Kernel Embedding of Distributions in Tensor Product Space}
We resort to kernel embedding to represent distributions. The marginal distributions of $X$ and $Y$ of each group are embedded in tensor product space $\cal{X}\otimes\cal{X}$ and $\cal{Y}\otimes\cal{Y}$, respectively. For simplicity, we use $\mathcal{H}$ to represent the tensor product space $\cal{X}\otimes\cal{X}$ and $\mathcal{G}$ to represent $\cal{Y}\otimes\cal{Y}$ in subsequent sections. Following the definition of kernel mean embedding, we define the mean embeddings of $X$ and $Y$ of group $i$ in $\mathcal{H}$ and $\mathcal{G}$ as:
\begin{dfnt}[tensor mean embedding]\label{def2}
	\begin{align}
	\mu^{(i)}_{\otimes X} \coloneqq \int{ \phi(x)\otimes\phi(x)p^{(i)}(x)dx }, \quad 
	\mu^{(i)}_{\otimes Y} \coloneqq \int{ \phi(y)\otimes\phi(y)p^{(i)}(y)dy }.	
	\label{1_4}
	\end{align}
	where $\phi(x)$ is the feature map of $x$ and $p^{(i)}(x)$ is the density of $x$ in group $i$. Similar notations go for $y$. 
\end{dfnt}
Definition \ref{def2} is the embedding of marginal densities of each group. Since our analysis is conducted on the base and variation of density of each group, we further define the tensor mean embedding of the base and variation of densities:
\begin{dfnt}[tensor mean embedding of the base and variation of distributions]\label{def3}
	\begin{align}
	\overline{\mu}_{\otimes X} &\coloneqq \int{ \phi(x)\otimes\phi(x)\overline{p}(x)dx }, \\
	\Delta \mu^{(i)}_{\otimes X} &\coloneqq \int{ \phi(x)\otimes\phi(x)\Delta p^{(i)}(x)dx }.
	\label{1_5}
	\end{align}
\end{dfnt}
$\overline{\mu}_{\otimes X}$ is the same in all groups and we have $\mu^{(i)}_{\otimes X} = \overline{\mu}_{\otimes X} + \Delta \mu^{(i)}_{\otimes X}$ from definitions \ref{def2} and \ref{def3}. Similarly, there is $\mu^{(i)}_{\otimes Y} = \overline{\mu}_{\otimes Y} + \Delta \mu^{(i)}_{\otimes Y}$. Definition \ref{def2} and \ref{def3} together state how marginal distributions are embedded in the tensor product space after decomposition. Next, we show how we make use of these tensor mean embeddings to infer the causal direction between $X$ and $Y$. To avoid analyzing probability densities directly, we substitute equation \ref{1_3} into definition \ref{def2} to conduct analysis on their embeddings:
\begin{align}
& ~\mu^{(i)}_{\otimes Y} \nonumber \\
= & \int{ \phi(y)\otimes\phi(y) \left[ \int{ (\overline{p}(y|x) + \Delta p^{(i)}(y|x) ) (\overline{p}(x) + \Delta p^{(i)}(x)) dx} \right] dy } \nonumber \\
= & \int{  \left[ \int{ \phi(y)\otimes\phi(y) \left(\overline{p}(y|x) + \Delta p^{(i)}(y|x) \right)  dy} \right] \left(\overline{p}(x) + \Delta p^{(i)}(x) \right) dx } \nonumber \\
\approx &  \int{  \left[ \int{ \phi(y)\otimes\phi(y) \overline{p}(y|x)  dy} \right] \overline{p}(x) dx } + \int{  \left[  \int{ \phi(y)\otimes\phi(y) \Delta p^{(i)}(y|x)  dy} \right] \overline{p}(x) dx } \nonumber \\
& + \int{  \left[ \int{ \phi(y)\otimes\phi(y) \overline{p}(y|x)  dy} \right] \Delta p^{(i)}(x) dx },
\label{1_6}
\end{align}
where we omit the term $\int{  \left(  \int{ \phi(y)\otimes\phi(y) \Delta p^{(i)}(y|x)  dy} \right) \Delta p^{(i)}(x) dx }$. Since the ranges of variables are usually bounded and distributions usually change smoothly instead of drastically in real-world situations, we consider it reasonable to omit the one with two variation terms. Although there exits sets of densities in which the omitted term of certain group would have magnitude comparable to the sum of the remaining three terms when the distribution shifts drastically, we deem it less likely to occur in real situations. Note that this claim is close in spirit to an assumption in \cite{zhang2015discovery} in which the authors assume the nonstationarity can be written as smooth functions of time or domain index. With this claim, we have the tensor mean embedding of the base of distributions as: 
\begin{align}
& ~\overline{\mu}_{\otimes Y} \nonumber \\
= & \int{ \phi(y)\otimes \phi(y) \overline{p}(y)dy } \nonumber \\
= & \int{ \phi(y)\otimes \phi(y) \left[ \frac{1}{N} \sum_{j=1}^{N}p^{(j)}(y) \right] dy } \nonumber \\
= & \int{ \phi(y)\otimes \phi(y) \left[\int \frac{1}{N} \sum_{j=1}^{N} \left(\overline{p}(y|x) + \Delta p^{(j)}(y|x) \right) \left( \overline{p}(x) + \Delta p^{(j)}(x) \right) dx\right] dy } \nonumber \\
\approx & \int{ \phi(y)\otimes \phi(y) \left[\int \frac{1}{N} \sum_{j=1}^{N} \left(\overline{p}(y|x)\overline{p}(x) + \overline{p}(y|x)\Delta p^{(j)}(x) + \Delta p^{(j)}(y|x)\overline{p}(x) \right) dx\right] dy } \nonumber \\
= &\int{ \phi(y)\otimes \phi(y) \left[ \int \overline{p}(y|x)\overline{p}(x)dx\right] dy },
\end{align}
where the approximately equal mark is again derived by omitting the one with two variation terms and the last equality is directly derived from the property shown in equation \ref{1_2}. Then we have the tensor mean embedding of the variation of distributions as:
\begin{align}
\Delta\mu^{(i)}_{\otimes Y} = \mu^{(i)}_{\otimes Y} - \overline{\mu}_{\otimes Y} \approx & \int{  \left(  \int{ \phi(y)\otimes\phi(y) \Delta p^{(i)}(y|x)  dy} \right) \overline{p}(x) dx } \nonumber \\
&+ \int{  \left( \int{ \phi(y)\otimes\phi(y) \overline{p}(y|x)  dy} \right) \Delta p^{(i)}(x) dx } ,
\label{1_delta}
\end{align}
which shows the relation between the tensor mean embedding of the variation of the effect and cause.
$\int{ \phi(y)\otimes\phi(y) \overline{p}(y|x)  dy}$ and $\int{ \phi(y)\otimes\phi(y) \Delta p^{(i)}(y|x)  dy}$ are matrices of functions of $X$. In addition, they are both symmetric and positive definite so they admit decomposition:
\begin{align}
&\int{ \phi(y)\otimes\phi(y) \overline{p}(y|x)  dy} = V(X)V^{T}(X) = \sum_{j=1}^{N_{H}} {v_{j}(X) v^{T}_{j}(X) }, \\
&\int{ \phi(y)\otimes\phi(y) \Delta p^{(i)}(y|x)  dy} = \Delta U(X) \Delta U^{T}(X) =  \sum_{j=1}^{N_{H}} {\Delta u^{(i)}_{j}(X) {\Delta u^{(i)}}^{T}_{j}(X) },
\label{1_7}
\end{align}
where $V(X)$ and $\Delta U(X)$ are lower triangular matrices, $v_{j}(X)$ and $\Delta u^{(i)}_{j}(X)$ denote the $j$-th column of $V(X)$ and $\Delta U(X)$, respectively; and $N_{H}$ denotes the dimension of $V(X)$. The symbol $\Delta$ indicates the corresponding relation of $\Delta U(X)$ to the variation of densities. By assuming that $v_{j}(X)$ and $\Delta u^{(i)}_{j}(X), j = 1, \dots, N_{H}$ lie in the space of $\phi(X)$, we have $v_{j}(X)  = \mathcal{A}_{j} \phi(X)$ and $\Delta u^{(i)}_{j}(X)  = \Delta \mathcal{B}^{(i)}_{j} \phi(X)$. $\mathcal{A}_{j}$ and $\Delta \mathcal{B}^{(i)}_{j}$ are matrices containing coefficient mapping from $\phi(X)$ to $v_{j}(X)$ and $\Delta u^{(i)}_{j}(X)$, respectively. Then we have
\begin{align}
&\int{ \phi(y)\otimes\phi(y) \overline{p}(y|x)  dy} = \sum_{j=1}^{N_{H}} \mathcal{A}_{j} \phi(X) \otimes  \phi(X) \mathcal{A}^{T}_{j}, \label{uyx} \\
&\int{ \phi(y)\otimes\phi(y) \Delta p^{(i)}(y|x)  dy} = \sum_{j=1}^{N_{H}} \Delta \mathcal{B}^{(i)}_{j} \phi(X) \otimes  \phi(X) {\Delta \mathcal{B}^{(i)}_{j}}^{T}. \label{duyix}
\end{align}
By substituting equation \ref{uyx} and \ref{duyix} into equation \ref{1_delta}, we further obtain
\begin{align}
\Delta \mu^{(i)}_{\otimes Y} & \approx \sum_{j=1}^{N_{H}} \mathcal{A}_{j} \Delta \mu^{(i)}_{\otimes X} \mathcal{A}^{T}_{j} + \sum_{j=1}^{N_{H}} \Delta \mathcal{B}^{(i)}_{j} \overline{\mu}_{\otimes X} {\Delta \mathcal{B}^{(i)}_{j}}^{T},
\label{1_8}
\end{align}
where $\Delta \mu^{(i)}_{\otimes X}$ and $\overline{\mu}_{\otimes X}$ are substituted in according to definition \ref{def2} and \ref{def3}.

\subsection{Inferring Causal Directions}
In this section, we discuss how we infer the causal direction using the kernel embedding of decomposed densities. Note again that we consider the case $X \to Y$ without loss of generality throughout this letter.

We start by taking normalized trace $\tau$ on both sides of equation \ref{1_8},
\begin{align}
\tau \left( \Delta\mu^{(i)}_{\otimes Y} \right) &\approx \tau \left( \sum_{j=1}^{N_{H}} \mathcal{A}_{j} \Delta \mu^{(i)}_{\otimes X} \mathcal{A}^{T}_{j} \right) + \tau \left( \sum_{j=1}^{N_{H}} \Delta \mathcal{B}^{(i)}_{j} \overline{\mu}_{\otimes X} {\Delta \mathcal{B}^{(i)}_{j}}^{T} \right) \nonumber \\
&= \tau \left( \sum_{j=1}^{N_{H}} \mathcal{A}^{T}_{j} \mathcal{A}_{j} \Delta \mu^{(i)}_{\otimes X} \right) + \tau \left( \sum_{j=1}^{N_{H}} {\Delta \mathcal{B}^{(i)}_{j}}^{T} \Delta \mathcal{B}^{(i)}_{j} \overline{\mu}_{\otimes X} \right) \nonumber \\
&= \tau \left( \mathcal{A} \Delta \mu^{(i)}_{\otimes X} \right) + \tau \left( {\Delta \mathcal{B}^{(i)}}  \overline{\mu}_{\otimes X} \right),
\label{1_10}
\end{align}
where $\tau(A) = \tr(A) / l_{A}$ is called the normalized trace of $A$, $l_{A}$ is the size of $A$, $\mathcal{A} = \sum_{j=1}^{N_{H}} \mathcal{A}^{T}_{j} \mathcal{A}_{j}$ and $\Delta \mathcal{B}^{(i)} = \sum_{j=1}^{N_{H}} {\Delta \mathcal{B}^{(i)}_{j}}^{T} \Delta \mathcal{B}^{(i)}_{j}$. Since the independence of the two mechanisms in Postulate \ref{post1} is difficult to quantify, we consider to use the density of the cause and the conditional density of the effect given the cause to represent the two mechanisms and adopt the independence between the base and variation of these two densities to infer the causal direction. The independence we rely on is based on the concept of free independence \citep{voiculescu1992free, voiculescu1997free}.
\begin{dfnt}[Free independence]\label{def4}
	\citep{voiculescu1992free, voiculescu1997free}. Let $\mathcal{D}$ be an algebra and $\psi: \mathcal{D} \rightarrow \mathbb{R}$ a linear functional on $\mathcal{D}$ with $\psi(1) = 1$. Then $A$ and $B$ are called free if
	\begin{align}
	\psi \left( p_{1}(A)q_{1}(B)p_{2}(A)q_{2}(B)\cdots \right) = 0,
	\end{align}
	for polynomials $p_{i}$, $q_{i}$, whenever $p_{i}(A) = q_{i}(B) = 0$.
\end{dfnt}

It is straightforward from definition \ref{def4} that if $A$ and $B$ are free independent, it holds that $\psi(AB) = \psi(A)\psi(B)$ \citep{voiculescu1992free, voiculescu1997free}. Then we have the following two assumptions to characterize the independence in postulate \ref{post1}:
\begin{assp}\label{assump1}
	We assume that the tensor mean embedding of the variation of marginal density of the cause ($\Delta \mu^{(i)}_{\otimes X}, i=1, \dots, N$) and $\mathcal{A}$ is free independent, and the tensor mean embedding of the base of marginal density of the cause ($\overline{\mu}_{\otimes X}$) and $\Delta \mathcal{B}^{(i)}, i=1, \dots, N$, is free independent, that is,
	\begin{align}
	\tau\left( \mathcal{A} \Delta \mu^{(i)}_{\otimes X}\right) = \tau\left( \mathcal{A} \right) \tau\left( \Delta \mu^{(i)}_{\otimes X} \right), i = 1, \dots, N, \label{assp_eq1} \\
	\tau\left( \Delta \mathcal{B}^{(i)}  \overline{\mu}_{\otimes X} \right) = \tau\left( \Delta \mathcal{B}^{(i)} \right) \tau\left(\overline{\mu}_{\otimes X} \right), i = 1, \dots, N, \label{assp_eq2}
	\end{align}
	where $N$ is the number of groups.
\end{assp}
Assumption \ref{assump1} captures the independence between the mechanism generating the cause and the mechanism mapping the cause to the effect. In equation \ref{assp_eq1}, $\mathcal{A}$ depends only on the base of the conditional densities $\overline{p}(Y|X)$ which corresponds to the second mechanism, and $\Delta \mu^{(i)}_{\otimes X}$ depends only on the variation of the marginal densities of the cause $\Delta p^{(i)}(X)$, which corresponds to the first mechanism. Therefore, the free independence between them characterizes the independence in postulate 1. Similarly, we have assumptions shown in equation \ref{assp_eq3}.
\begin{assp}\label{assump2}
	Regarding the normalized trace of the tensor mean embedding of variation of marginal densities of the cause in each group as a realization of a random variable $\tau_{\Delta \mu_{\otimes X}}$ and each $\tau \left(\Delta \mathcal{B}^{(i)} \right)$ as a realization of another random variable $\tau_{\Delta \mathcal{B}}$, we assume that these two random variables are independent, i.e.
	\begin{align}
	\tau_{\Delta \mu_{\otimes X}} \perp\!\!\!\perp \tau_{\Delta \mathcal{B}}.
	\label{assp_eq3}
	\end{align}
\end{assp}
Assumption \ref{assump2} is also motivated by the independence in postulate \ref{post1}. Specifically, $\tau_{\Delta \mu_{\otimes X}}$ captures the information of the variation of marginal densities of the cause, and $\tau_{\Delta \mathcal{B}}$ captures the information of the variation of conditional densities. We interpret postulate \ref{post1} as the independence between the marginal and conditional. Therefore, this independence between their variations of densities (approximately) holds. With assumption \ref{assump1}, equation \ref{1_10} becomes
\begin{align}
\tau\left(\Delta\mu^{(i)}_{\otimes Y}\right) \approx \tau\left( \mathcal{A} \right) \tau\left( \Delta \mu^{(i)}_{\otimes X} \right) + \tau\left( \Delta \mathcal{B}^{(i)} \right) \tau\left(\overline{\mu}_{\otimes X} \right).
\label{1_11}
\end{align}
Since $\overline{p}(x)$ and $\overline{p}(y|x)$ are fixed given $\mathbf{G}$, $\tau\left(\overline{\mu}_{\otimes X}\right)$ and $\tau\left(\mathcal{A}\right)$ are the same in all groups. We introduce the following notations for simplicity:

\noindent
\textbf{Notation 1.} For any $G_{i} \in \mathbf{G}$, we use $\tau_{x}^{(i)}$ and $\tau_{y}^{(i)}$
to represent $\tau\left(\Delta\mu^{(i)}_{\otimes X}\right)$ and $\tau\left(\Delta\mu^{(i)}_{\otimes Y}\right)$, respectively. $\epsilon_{x\rightarrow y}^{(i)}$ denotes $\tau\left(\Delta \mathcal{B}^{(i)}  \overline{\mu}_{\otimes X} \right)$, which is the corresponding noise term. $c_{y|x}$ denotes $\tau\left(\mathcal{A}\right)$. We view each $\tau_{x}^{(i)}$ as a realization of a random variable $\tau_{x}$. Similarly, we have $\tau_{y}$ and $\epsilon_{x\rightarrow y}$.
\begin{prop}\label{prop1}
	If the causal direction is $X\to Y$ and assumptions \ref{assump1} and \ref{assump2} hold, the normalized trace of the tensor mean embeddings of the variation of the densities of the cause $\left(\tau_{x}\right)$ and the effect $\left(\tau_{y}\right)$ fulfill the following linear nongaussian additive model \citep{kano2003causal}:
	\begin{align}
	\tau_{y} &\approx c_{y|x} \tau_{x}+ \epsilon_{x\rightarrow y}.
	\label{lnr}
	\end{align}
\end{prop}
\begin{proof}
	By adopting notations in notation 1, equation \ref{1_11} becomes 
	\begin{align}
	\tau_{y}^{(i)} \approx c_{y|x} \tau_{x}^{(i)}+ \epsilon_{x\rightarrow y}^{(i)},\quad i = 1, \dots, N.
	\end{align}
	We first show that $\epsilon_{x\rightarrow y}$ follows nongaussian distributions. According to assumption 1, we have
	\begin{align}
	\epsilon_{x\rightarrow y}^{(i)} = \tau\left( \Delta \mathcal{B}^{(i)} \right) \tau\left(\overline{\mu}_{\otimes X} \right),
	\end{align}
	where $\tau\left(\overline{\mu}_{\otimes X} \right)$ is fixed and thus can be viewed as a constant. From the definition of $\Delta \mathcal{B}^{(i)}$ we have
	\begin{align}
	\tau\left( \Delta \mathcal{B}^{(i)} \right) = \frac{1}{N_{H}} \tr\left( \sum_{j=1}^{N_{H}} {\Delta \mathcal{B}^{(i)}_{j}}^{T} \Delta \mathcal{B}^{(i)}_{j} \right) = \frac{1}{N_{H}} \sum_{j=1}^{N_{H}} \tr\left( {\Delta \mathcal{B}^{(i)}_{j}}^{T} \Delta \mathcal{B}^{(i)}_{j} \right).
	\end{align}
	Since $\tr\left( {\Delta \mathcal{B}^{(i)}_{j}}^{T} \Delta \mathcal{B}^{(i)}_{j} \right)$ are positive for all $j$, we have $\tau\left( \Delta \mathcal{B}^{(i)} \right) > 0$. Therefore, the distribution of $\epsilon_{x\rightarrow y}^{(i)}$ is not symmetric and is thus not Gaussian distributed.
	
	Second, we have $\tau_{x}$ is independent of $\epsilon_{x\rightarrow y}$ according to the independence between $\tau_{\Delta\mu_{\otimes X}}$ and $\tau_{\Delta \mathcal{B}}$ in assumption \ref{assump2}. Then we conclude equation \ref{lnr} forms a linear non-Gaussian additive model.
\end{proof}
According to the identifiability of LiNGAM~\citep{kano2003causal,shimizu-et-al:lingam}, $\tau_{y}$ and $\epsilon_{y\rightarrow x}$ are dependent. By exploiting the cause-effect asymmetry that the cause is independent of the noise only in the causal direction, we propose the following causal inference approach: embedding-based nonstationary causal model inference (ENCI).

\noindent
\textbf{Causal Inference Approach (ENCI)}: Given data set $\mathbf{G}$, we compute $\tau^{(i)}_{x}$ and $\tau^{(i)}_{y}$ for $i = 1, \dots, N$ and conclude that $X\to Y$ if $\tau_{x} \perp\!\!\!\perp \epsilon_{x\rightarrow y}$, otherwise $Y \to X$ if $\tau_{y} \perp\!\!\!\perp \epsilon_{y\rightarrow x}$.

Hilbert Schimidt Independence Criterion (HSIC)~\citep{gretton2007kernel} is applied to measure the independence between the regressor and its corresponding noise on both hypothetical directions, and we favor the direction with less dependence in practice. The ENCI algorithm is given in algorithm~\ref{alg1}.
\begin{algorithm}[ht]
	\caption{ENCI for cause-effect pairs}
	\label{alg1}
	\renewcommand{\algorithmicrequire}{\textbf{Input:}}
	\renewcommand{\algorithmicensure}{\textbf{Output:}}
	\begin{algorithmic}[1]
		\REQUIRE $N$ data groups $\mathbf{G} = \{G_{1}, G_{2}, \dots, G_{N}\}$
		\ENSURE The causal direction
		\STATE Normalize $X$ and $Y$ in each group;
		\STATE Compute $\tau_{x}^{(i)}$ and $\tau_{y}^{(i)}$ for $i=1, \dots, N$;
		\STATE Compute residual $\epsilon_{x\rightarrow y}$ and $\epsilon_{y\rightarrow x}$ by conducting least square regressions;
		\STATE Apply HSIC on $\tau_{x}$ and $\epsilon_{x\rightarrow y}$, denote the quotient of testStat and thresh returned by HSIC by $r_{x\rightarrow y}$. Similarly we have $r_{y\rightarrow x}$.
		\IF{$ r_{x\rightarrow y} < r_{y\rightarrow x}$}
			\STATE The causal direction is $x \rightarrow y$;
		\ELSIF{$ r_{x\rightarrow y} > r_{y\rightarrow x}$}
			\STATE The causal direction is $y \rightarrow x$;
		\ELSE
			\STATE No decision made.
		\ENDIF
	\end{algorithmic}
\end{algorithm}

\subsection{Empirical Estimations}
In this section, we show how to estimate $\tau_{x}^{(i)}$ and $\tau_{y}^{(i)}$ for $i=1, \dots, N$ based on the observations. 

Let $\mathbf{\Phi}^{(i)} = \left[ \phi(x_{1}^{(i)}), \dots, \phi(x_{N_{i}}^{(i)}) \right]$ and $\mathbf{\Gamma}^{(i)} = \left[ \gamma(x_{1}^{(i)}), \dots, \gamma(x_{N_{i}}^{(i)}) \right]$ be the feature matrices of $X$ and $Y$ in group $i$, respectively, given observations in $\mathbf{G}$. We estimate the mean embedding of $p^{(i)}(X)$ in $\mathcal{X}\otimes \mathcal{X}$ as
\begin{align}
\hat{\mu}_{\otimes X}^{(i)} = \frac{1}{N_{i}} \Phi^{(i)} H \left( \Phi^{(i)} H\right)^{T},
\end{align}
where $N_{i}$ is the number of observations in $i$th group, $H = I - \frac{1}{N_{i}} \mathbf{1} \mathbf{1}^{T}$ and $\mathbf{1}$ is a column vector of all 1s. Since we have 
\begin{align}
\hat{\overline{\mu}}_{\otimes X} & = \int{ \phi(x)\otimes \phi(x) \hat{\overline{p}}(x)dx } \nonumber \\
& = \int{ \phi(x)\otimes \phi(x)  \left[ \frac{1}{N} \sum_{j=1}^{N}\hat{p}^{(j)}(x) \right] dx } \nonumber \\
& = \frac{1}{N} \sum_{j=1}^{N}{ \int{ \phi(x)\otimes \phi(x) \hat{p}^{(j)}(x) dx } } \nonumber \\
& = \frac{1}{N} \sum_{j=1}^{N} {\hat{\mu}_{\otimes X}^{(j)}},
\end{align}
for estimating the tensor mean embedding of the base of distributions $\overline{\mu}_{\otimes X}$, the tensor mean embedding of the variation of distributions $\Delta \mu_{\otimes X}^{(i)}$ is estimated as
\begin{align}
\Delta \hat{\mu}_{\otimes X}^{(i)} = \hat{\mu}_{\otimes X}^{(i)} - \hat{\overline{\mu}}_{\otimes X} = \hat{\mu}_{\otimes X}^{(i)} - \frac{1}{N} \sum_{j=1}^{N} {\hat{\mu}_{\otimes X}^{(j)}}.
\label{est_delta}
\end{align}
By taking the normalized trace on both sides of equation \ref{est_delta}, we have
\begin{align}
\tau_{x}^{(i)} & = \tau\left( \hat{\mu}_{\otimes X}^{(i)} \right) - \tau\left( \frac{1}{N} \sum_{i=1}^{N} {\hat{\mu}_{\otimes X}^{(i)}} \right) \nonumber \\
& \approx \tau\left( \frac{1}{N_{i}} \mathbf{\Phi}^{(i)} H \left( \mathbf{\Phi}^{(i)} H\right)^{T} \right) - \frac{1}{N} \sum_{j=1}^{N} \tau\left( \frac{1}{N_{j}} \mathbf{\Phi}^{(j)} H \left( \mathbf{\Phi}^{(j)} H\right)^{T} \right) \nonumber \\
& = \frac{1}{N_{i}^{2}} \tr\left( K^{(i)}_{x} H \right) - \frac{1}{N} \sum_{j=1}^{N} \left[ \frac{1}{N_{j}^{2}} \tr\left( K^{(j)}_{x} H \right) \right],
\end{align}
where $N$ is the total number of groups, $N_{i}$ is the number of observations in $i$th group and $K^{(i)}_{x} = \left(\mathbf{\Phi}^{(i)}\right)^{T}\mathbf{\Phi}^{(i)}$ is the kernel matrix of $X$ in $i$th group. Similarly, we have
\begin{align}
	\tau_{y}^{(i)} = = \frac{1}{N_{i}^{2}} \tr\left( K^{(i)}_{y} H \right) - \frac{1}{N} \sum_{j=1}^{N} \left[ \frac{1}{N_{j}^{2}} \tr\left( K^{(j)}_{y} H \right) \right],
\end{align}
where $K^{(i)}_{y} = \left(\mathbf{\Gamma}^{(i)}\right)^{T}\mathbf{\Gamma}^{(i)}$ is the kernel matrix of $X$ in $i$th group. We can see that both $\tau_{x}^{(i)}$ and $\tau_{y}^{(i)}$ can be easily calculated from Gram matrix using kernel methods.

\section{Extending ENCI to Causal Graph Discovery}
In this section, we extend ENCI to causal discovery for two kinds of directed acyclic graphs (DAGs). One is a tree-structured graph in which each node has at most one parent node. The other is multiple-independent-parent graph in which parent nodes of each node are mutually independent. Examples of these two kinds of DAGs are shown in Figure~\ref{fig:exmp}.
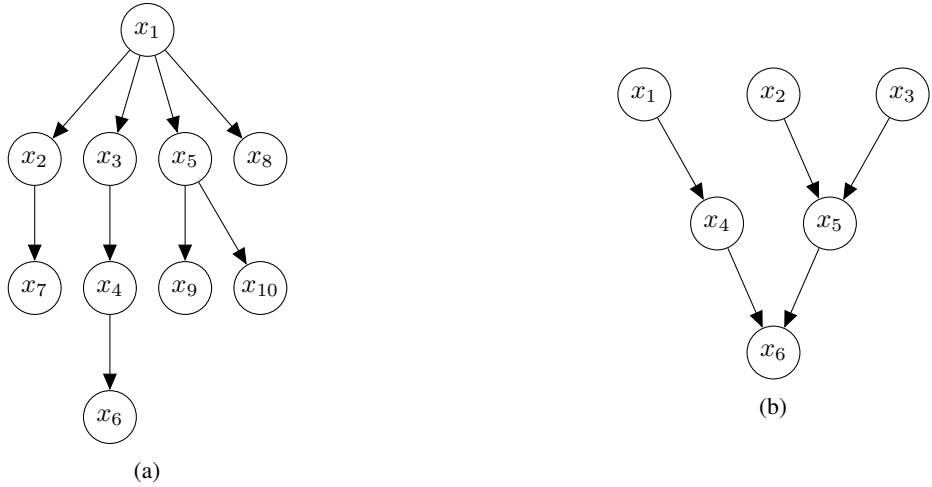
\begin{figure}[!ht]
	\begin{subfigure}{.5\textwidth}
		\centering
		\begin{tikzpicture}[scale=1.0,transform shape,
		state/.style={circle,draw,thick,loop above,inner sep=0,minimum width=10}]
		\node[latent] (x1) {$x_{1}$} ; %
		\node[latent, below=of x1, xshift=-1.5cm] (x2) {$x_{2}$} ; %
		\node[latent, below=of x1, xshift=-0.5cm] (x3) {$x_{3}$} ; %
		\node[latent, below=of x1, xshift=0.5cm] (x5) {$x_{5}$} ; %
		\node[latent, below=of x1, xshift=1.5cm] (x8) {$x_{8}$} ; %
		\edge {x1}{x2, x3, x5, x8} ; %
		
		\node[latent, below=of x2] (x7) {$x_{7}$} ; %
		\node[latent, below=of x3] (x4) {$x_{4}$} ; %
		\node[latent, below=of x5] (x9) {$x_{9}$} ; %
		\node[latent, below=of x8] (x10) {$x_{10}$} ; %
		\edge {x2}{x7} ; %
		\edge {x3}{x4} ; %
		\edge {x5}{x9, x10} ; %
		
		\node[latent, below=of x4] (x6) {$x_{6}$} ; %
		\edge {x4}{x6} ; %
		\end{tikzpicture}
		\caption{}
		\label{fig:exmp_tree}
	\end{subfigure}
	\begin{subfigure}{.5\textwidth}
		\centering
		\begin{tikzpicture}[scale=1.0,transform shape,
		state/.style={circle,draw,thick,loop above,inner sep=0,minimum width=10}]
		\node[latent] (x1) {$x_{1}$} ; %
		\node[latent, right=of x1] (x2) {$x_{2}$} ; %
		\node[latent, right=of x2] (x3) {$x_{3}$} ; %
		\node[latent, below=of x2, xshift=-0.75cm] (x4) {$x_{4}$} ; %
		\node[latent, below=of x2, xshift=0.75cm] (x5) {$x_{5}$} ; %
		\node[latent, below=of x4, xshift=0.75cm] (x6) {$x_{6}$} ; %
		
		\edge{x1}{x4};
		\edge{x2, x3}{x5};
		\edge{x4, x5}{x6};
		\end{tikzpicture} 
		\caption{}
		\label{fig:exmp_mtp}
	\end{subfigure}
	\caption{Examples of (a) Tree-structured graph (b) Multiple-independent-parent graph.}
	\label{fig:exmp}
\end{figure}

\subsection{Describing Causal Relationship by Directed Acyclic Graphs}
Consider a finite set of random variables $\mathbf{X} = (X_{1}, \dots, X_{p})$ with index set $\mathbf{V} \coloneqq \{1, \dots, p\}$. A graph $\mathcal{G}=(\mathbf{V}, \mathbf{E})$ consists of nodes in $\mathbf{V}$ and edges $(m, n)$ in $\mathbf{E}$ for any $m, n \in \mathbf{V}$. Then we introduce graph terminologies required for subsequent sections. Most of the definitions are from \citep{spirtes2000causation}.

Edge $(m, n)$ is a directed link from node $m$ to node $n$. Node $m$ is called a parent of $n$, and $n$ is called a child of $m$ if $(m,n)\in \mathbf{E}$. The parent set of $n$ is denoted by $pa(n)$ and its child set by $ch(n)$. Nodes $m$, $n$ are called adjacent if either $(m, n)\in \mathbf{E}$ or $(n, m)\in \mathbf{E}$. A path in $\mathcal{G}$ is a sequence of distinct vertices $m_{1}, \dots, n_{q}$ such that $m_{k}$ and $n_{k+1}$ are adjacent for all $k=1, \dots, q-1$. If $(m_{k}, m_{k+1}) \in \mathbf{E}$ for all $k$, the path is also called a directed path from $m_{1}$ to $m_{q}$. $\mathcal{G}$ is called a partially directed acyclic graph (PDAG) if there is no directed cycle, i.e., there is no pair $(m, n)$ such that there are directed paths from $m$ to $n$ and from $n$ to $m$. $\mathcal{G}$ is called a directed acyclic graph (DAG) if it is a PDAG and all edges are directed.

General causal graph discovery is very challenging, especially when the relation between a variable pair is a complicated nonlinear stochastic process. In the following section, we show how we discover the causal structure tree-structured graphs (TSG) and multiple-independent-parents graph (MIPG). Note that the causal relation between a variable and its parent node in our model not only could be complicated nonlinear functions but also varies in different groups.
 
\subsection{Tree-Structured Causal Graph Discovery}
In a TSG $\mathcal{G}$ with $p$ nodes, each variable $X_{m}$ and its only parent node $pa(X_m)$ fulfill the linear relation in Equation \ref{lnr}. Thus we have the following proposition for TSG:
\begin{prop}
	In a TSG $\mathcal{G}$ where each variable $X_{m}$ has only one parent node, the normalized traces of the tensor mean embedding of the variation of densities of all variables $\left(\tau_{x_{m}},~ m = 1, \dots, p \right)$ fulfill a linear nongaussian acyclic model (LiNGAM) \citep{shimizu-et-al:lingam} if assumption \ref{assump1} and \ref{assump2} hold:
	\begin{align}
	\bm{\tau}_{x}&\approx \bm{C} \bm{\tau}_{x} + \bm{\epsilon},
	\end{align}
	where $\bm{\tau}_{x} = \left[\tau_{x_{1}}, \dots, \tau_{x_{p}} \right]^{T}$, coefficient matrix $\bm{C}$ whose element on $n$-th row and $m$-th column equals to $c_{x_{n}|x_{m}}$ could be permuted to a lower triangular matrix and $\bm{\epsilon} = \left[\epsilon_{pa(x_{1})\to x_{1}}, \dots, \epsilon_{pa(x_{p})\to x_{p}} \right]^{T}$ collects all noise terms $\epsilon_{pa(x_{m})\to x_{m}},~ m = 1, \dots, p$.
\end{prop}
\begin{proof}
	First, $\tau_{x_{m}}$, where $m = 1, \dots, p$, could be arranged in a causal order in which no later variable is the cause of earlier ones due to the acyclicity of the graph. Note that causal order in subsequent sections also means that this condition holds for a sequence of variables. Second, the noise term $\epsilon_{pa(x_{m})\rightarrow x_{m}}$, where $m = 1, \dots, p$, follows nongaussian distributions as shown in proposition \ref{prop1}. Thirdly, assumption \ref{assump2} ensures that $\tau_{x_{m}} \perp\!\!\!\perp \epsilon_{pa(x_{m})\rightarrow x_{m}}$ for $m = 1, \dots, p$.
\end{proof}
Therefore, the graph formed by $\tau_{x_{m}}$, where $m = 1, \dots, p$, fulfills the structure of LiNGAM~\citep{shimizu-et-al:lingam} so we can apply LiNGAM on $\tau_{x}$ to infer the causal structure of the causal graph consists of $X_{1}, \dots, X_{p}$.

\subsection{Multiple-Independent-Parent Graph Discovery}
We extend ENCI to cases where each node could have more than one parent node provided that all its parent nodes are mutually independent.

Suppose a variable $Y$ in graph $\mathcal{G}$ has $q$ independent parent nodes - $X_{1}, \dots, X_{q}$. The marginal density of $Y$ in group $i$ can be obtained from 
\begin{align}
p^{(i)}(Y) = \int{ p^{(i)}(Y|x_{1}, \dots, x_{q})p^{(i)}(x_{1}, \dots, x_{q})dx_{1}\cdots dx_{q} }.
\end{align}
Then by substituting $p^{(i)}(Y)$ into $\mu^{(i)}_{\otimes Y}$ with $p^{(i)}(Y|X_{1}, \dots, X_{q})$ decomposed as $p^{(i)}(Y|X_{1}, \dots, X_{q}) = \overline{p}(Y|X_{1}, \dots, X_{q}) + \Delta p^{(i)}(Y|X_{1}, \dots, X_{q})$ and integrating with respect to $Y$, we have
\begin{align}
\mu^{(i)}_{\otimes Y}=& \int \phi(y)\otimes\phi(y) \left( \int{ \overline{p}(y|x_{1}, \dots, x_{q})p^{(i)}(x_{1}, \dots, x_{q})dx_{1}\cdots dx_{q} } +  \nonumber \right. \\
& \left. \int{\Delta p^{(i)}(y|x_{1}, \dots, x_{q})p^{(i)}(x_{1}, \dots, x_{q})dx_{1}\cdots dx_{q}}\right) dy  \nonumber \\
= & \int{ \left( \int{\phi(y)\otimes\phi(y) \overline{p}(y|x_{1}, \dots, x_{q}) dy} \right) p^{(i)}(x_{1}, \dots, x_{q}) dx_{1}\cdots dx_{q} } + \nonumber \\ 
& \int{ \left( \int{\phi(y)\otimes\phi(y) \Delta p^{(i)}(y|x_{1}, \dots, x_{q}) dy} \right) p^{(i)}(x_{1}, \dots, x_{q}) dx_{1}\cdots dx_{q} }.
\label{mlt_1}
\end{align}
Following the same idea in the previous section, we conduct decomposition on both $ \int{\phi(y)\otimes\phi(y) \overline{p}(y|X_{1}, \dots, X_{q}) dy}$ and $\int{\phi(y)\otimes\phi(y) \Delta p^{(i)}(y|X_{1}, \dots, X_{q}) dy}$ and thus obtain 
\begin{align}
&\int{\phi(y)\otimes\phi(y) \overline{p}(y|X_{1}, \dots, X_{q}) dy} = \sum_{j=1}^{N_{H}} {v_{j}(X_{1}, \dots, X_{q}) v^{T}_{j}(X_{1}, \dots, X_{q}) }, \\
&\int{\phi(y)\otimes\phi(y) \Delta p^{(i)}(y|X_{1}, \dots, X_{q}) dy} = \sum_{j=1}^{N_{H}} {\Delta u^{(i)}_{j}(X_{1}, \dots, X_{q}) \Delta{u^{(i)}}^{T}_{j}(X_{1}, \dots, X_{q}) },
\end{align}
where $v_{j}(X_{1}, \dots, X_{q})$ denotes the $j$-th column of $\int{\phi(y)\otimes\phi(y) \overline{p}(y|X_{1}, \dots, X_{q}) dy}$ and $\Delta u^{(i)}_{j}(X_{1}, \dots, X_{q})$ denotes the $j$-th column of $\int{\phi(y)\otimes\phi(y) \Delta p^{(i)}(y|X_{1}, \dots, X_{q}) dy}$. By assuming that both $v_{j}(X_{1}, \dots, X_{q})$ and $\Delta u^{(i)}_{j}(X_{1}, \dots, X_{q}), j = 1, \dots, N_{H}$ lie in the space of feature map $\phi(X_{1}, \dots, X_{q})$, we have $v_{j}(X_{1}, \dots, X_{q})  = \mathcal{A}_{j} \phi(X_{1}, \dots, X_{q})$ and $\Delta u^{(i)}_{j}(X_{1}, \dots, X_{q})  = \Delta \mathcal{B}^{(i)}_{j} \phi(X_{1}, \dots, X_{q})$. Then they become
\begin{small}
	\begin{align}
	&\int{\phi(y)\otimes\phi(y) \overline{p}(y|X_{1}, \dots, X_{q}) dy} = \sum_{j=1}^{N_{H}} { \mathcal{A}_{j} \phi(X_{1}, \dots, X_{q}) \otimes \phi(X_{1}, \dots, X_{q}) \mathcal{A}^{T}_{j}}, \label{sum_g_1} \\
	&\int{\phi(y)\otimes\phi(y) \Delta p^{(i)}(y|X_{1}, \dots, X_{q}) dy} = \sum_{j=1}^{N_{H}} { \Delta \mathcal{B}^{(i)}_{j} \phi(X_{1}, \dots, X_{q}) \otimes \phi(X_{1}, \dots, X_{q}) {\Delta \mathcal{B}^{(i)}_{j}}^{T} }. \label{sum_g_2}
	\end{align}
\end{small}
By plugging in equations \ref{sum_g_1} and \ref{sum_g_2}, equation \ref{mlt_1} becomes
\begin{align}
& \mu^{(i)}_{\otimes Y} \nonumber \\
= & \int{ \left( \sum_{j=1}^{N_{H}} { \mathcal{A}_{j} \phi(x_{1}, \dots, x_{q}) \otimes \phi(x_{1}, \dots, x_{q}) \mathcal{A}^{T}_{j}} \right)  p^{(i)}(x_{1}, \dots, x_{q})dx_{1}\cdots dx_{q} } \nonumber \\
& + \int{ \left( \sum_{j=1}^{N_{H}} { \Delta \mathcal{B}^{(i)}_{j} \phi(x_{1}, \dots, x_{q}) \otimes \phi(x_{1}, \dots, x_{q}) {\Delta \mathcal{B}^{(i)}_{j}}^{T} } \right)  p^{(i)}(x_{1}, \dots, x_{q})dx_{1}\cdots dx_{q} } \nonumber \\
= & \sum_{j=1}^{N_{H}} \mathcal{A}_{j} \left[ \int{ \phi(x_{1}, \dots, x_{q}) \otimes \phi(x_{1}, \dots, x_{q})  p^{(i)}(x_{1}, \dots, x_{q})dx_{1}\cdots dx_{q} } \right] \mathcal{A}^{T}_{j} \nonumber \\
& + \sum_{j=1}^{N_{H}} \Delta \mathcal{B}^{(i)}_{j} \left[ \int{ \phi(x_{1}, \dots, x_{q}) \otimes \phi(x_{1}, \dots, x_{q})  p^{(i)}(x_{1}, \dots, x_{q}) dx_{1}\cdots dx_{q} } \right] {\Delta \mathcal{B}^{(i)}_{j}}^{T}.
\label{mu_y_eq_3}
\end{align}
Observing that there exists a common term of integration in each term of the summation in equation \ref{mu_y_sumi}, we now analyze this integral term in square brackets. Due to mutual independence among variables $X_{k}$ for $k=1,\dots,q$, $p^{(i)}(X_{1}, \dots, X_{q})$ admits the following factorization:
\begin{align}
p^{(i)}(X_{1}, \dots, X_{q})=p^{(i)}(X_{1}) \cdots p^{(i)}(X_{q}).
\end{align}
Then we adopt Bochner's theorem~\citep{rudin2011fourier} in analyzing $\phi(X_{1}, \dots, X_{q})$. Bochner's theorem states that a continuous shift-invariant kernel $K(x,y)=k(x-y)$ is a positive-definite function if and only if $k(t)$ is the Fourier transform of a nonnegative measure $\rho(\omega)$. Let $\alpha=\int{d\rho(\omega)}$, $p_{\omega}=\rho/\alpha$, and $\omega_{1},\omega_{2}, \dots, \omega_{k}$ be independent samples from $p_{\omega}$. Then the random projection vector $\phi(X)$ can be
\begin{align}
\phi(X) = \frac{\alpha}{\sqrt[]{k}}\left[ e^{-i\omega_{1}^TX}, \dots, e^{-i\omega_{k}^TX} \right].
\end{align}
Similarly, we have 
\begin{align}
\phi(X_{1}, \dots, X_{n})=\frac{\alpha}{\sqrt[]{k}}\left[ e^{-i\left(\omega_{11}^TX_{1}+\cdots + \omega_{1n}^TX_{n}\right)}, \dots, e^{-i\left(\omega_{k1}^TX_{1}+\cdots+ \omega_{kn}^TX_{n}\right)} \right],
\end{align}
which leads to
\begin{align}
\phi(X_{1}, \dots, X_{q})&=\phi(X_{1})\circ\dots\circ\phi(X_{q}),
\end{align}
where $\phi(X_{j})\circ\phi(X_{k})$ denotes the element-wise product. Since
\begin{align}
&\left(\phi(X_{1})\circ\dots\circ\phi(X_{q})\right)\otimes\left(\phi(X_{1})\circ\dots\circ\phi(X_{q})\right) \nonumber \\
= &\left(\phi(X_{1})\otimes\phi(X_{1})\right)\circ\dots\circ\left(\phi(X_{q})\otimes\phi(X_{q})\right),
\end{align} 
the integration in equation \ref{mu_y_eq_3} becomes
\begin{align}
&\int{ \phi(x_{1}, \dots, x_{q}) \otimes \phi(x_{1}, \dots, x_{q})  p^{(i)}(x_{1}, \dots, x_{q}) dx_{1}\cdots dx_{q} } \nonumber \\
= & \int{ \left(\phi(x_{1})\otimes\phi(x_{1})\right)\circ\dots\circ\left(\phi(x_{q})\otimes\phi(x_{q})\right) ~ p^{(i)}(x_{1})\cdots p^{(i)}(x_{q}) ~ dx_{1}\cdots dx_{q}} \nonumber \\
= & \int{ \phi(x_{1})\otimes\phi(x_{1}) \left(p(x_{1}) + \Delta p^{(i)}(x_{1})\right)dx_{1}} \circ \dots \nonumber \\
& \dots \circ \int{ \phi(x_{q})\otimes\phi(x_{q}) \left(p(x_{q}) + \Delta p^{(i)}(x_{q})\right) dx_{q}} \nonumber \\
= & \left(\overline{\mu}_{\otimes X_{1}} + \Delta\mu_{\otimes X_{1}}^{(i)}\right) \circ\dots\circ\left(\overline{\mu}_{\otimes X_{q}} + \Delta\mu_{\otimes X_{q}}^{(i)}\right).
\label{mu_y_sumi}
\end{align}
By substituting equation \ref{mu_y_sumi} into equation \ref{mu_y_eq_3} we have
\begin{align}
\mu^{(i)}_{\otimes Y}&= \sum_{j=1}^{N_{H}} \mathcal{A}_{j} \left[ \left(\overline{\mu}_{\otimes X_{1}} + \Delta\mu_{\otimes X_{1}}^{(i)}\right) \circ\dots\circ\left(\overline{\mu}_{\otimes X_{q}} + \Delta\mu_{\otimes X_{q}}^{(i)}\right) \right] \mathcal{A}^{T}_{j} \nonumber \\
& \quad + \sum_{j=1}^{N_{H}} \Delta \mathcal{B}^{(i)}_{j} \left[ \left(\overline{\mu}_{\otimes X_{1}} + \Delta\mu_{\otimes X_{1}}^{(i)}\right) \circ\dots\circ\left(\overline{\mu}_{\otimes X_{q}} + \Delta\mu_{\otimes X_{q}}^{(i)}\right) \right] {\Delta \mathcal{B}^{(i)}_{j}}^{T} \nonumber \\
&\approx \sum_{j=1}^{N_{H}} \mathcal{A}_{j} \left[\left( \overline{\mu}_{\otimes X_{1}}\circ\dots\circ \overline{\mu}_{\otimes X_{q}} \right) + \left(\Delta \mu_{\otimes X_{1}}^{(i)}\circ\dots\circ \overline{\mu}_{\otimes X_{q}} \right) + \cdots \right. \nonumber \\
& \quad + \left. \left(\overline{\mu}_{\otimes X_{1}}\circ\dots\circ \Delta\mu_{\otimes X_{q}}^{(i)} \right) \right] \mathcal{A}^{T}_{j} + \sum_{j=1}^{N_{H}} \Delta \mathcal{B}^{(i)}_{j} \left[ \overline{\mu}_{\otimes X_{1}}\circ\dots\circ \overline{\mu}_{\otimes X_{q}} \right] {\Delta \mathcal{B}^{(i)}_{j}}^{T},
\end{align}
where we omit terms with more than one tensor mean embedding of variation of densities . Following the same idea in equation \ref{1_delta}, we compute the variation of tensor embedding of $Y$ by
\begin{align}
\Delta \mu^{(i)}_{\otimes Y} &\approx \mu^{(i)}_{\otimes Y} - \overline{\mu}_{\otimes Y} \nonumber \\
&= \sum_{j=1}^{N_{H}} \mathcal{A}_{j} \left[\left(\Delta \mu_{\otimes X_{1}}^{(i)}\circ\dots\circ \overline{\mu}_{\otimes X_{q}} \right) + \cdots + \left(\overline{\mu}_{\otimes X_{1}}\circ\dots\circ \Delta\mu_{\otimes X_{q}}^{(i)} \right) \right] \mathcal{A}^{T}_{j} \nonumber \\
& \quad + \sum_{j=1}^{N_{H}} \Delta \mathcal{B}^{(i)}_{j} \left[ \overline{\mu}_{\otimes X_{1}}\circ\dots\circ \overline{\mu}_{\otimes X_{q}} \right] {\Delta \mathcal{B}^{(i)}_{j}}^{T}.
\label{bf_tau}
\end{align}
Then by taking normalized trace on both sides of equation \ref{bf_tau} we have
\begin{align}
\tau \left( \Delta \mu^{(i)}_{\otimes Y} \right) &\approx \tau \left( \sum_{j=1}^{N_{H}} \mathcal{A}^{T}_{j} \mathcal{A}_{j} \left[\left(\Delta \mu_{\otimes X_{1}}^{(i)}\circ\dots\circ \overline{\mu}_{\otimes X_{q}} \right) + \cdots + \left(\overline{\mu}_{\otimes X_{1}}\circ\dots\circ \Delta\mu_{\otimes X_{q}}^{(i)} \right) \right]  \right) \nonumber \\
& \quad + \tau \left( \sum_{j=1}^{N_{H}} {\Delta \mathcal{B}^{(i)}_{j}}^{T} \Delta \mathcal{B}^{(i)}_{j} \left[ \overline{\mu}_{\otimes X_{1}}\circ\dots\circ \overline{\mu}_{\otimes X_{q}} \right]  \right) \nonumber \\
&= \tau \left( \mathcal{A} \left[\left(\Delta \mu_{\otimes X_{1}}^{(i)}\circ\dots\circ \overline{\mu}_{\otimes X_{q}} \right) + \cdots + \left(\overline{\mu}_{\otimes X_{1}}\circ\dots\circ \Delta\mu_{\otimes X_{q}}^{(i)} \right) \right]  \right) \nonumber \\
& \quad + \tau \left( \Delta \mathcal{B}^{(i)} \left[ \overline{\mu}_{\otimes X_{1}}\circ\dots\circ \overline{\mu}_{\otimes X_{q}} \right]  \right) \nonumber \\
&= \tau\left(\mathcal{A}\right)\tau\left(\Delta \mu_{\otimes X_{1}}^{(i)}\circ\dots\circ \overline{\mu}_{\otimes X_{q}} \right)+\dots +\tau\left(\mathcal{A}\right)\tau\left(\overline{\mu}_{\otimes X_{1}}\circ\dots\circ \Delta\mu_{\otimes X_{q}}^{(i)}\right) \nonumber \\
&\quad +\tau\left(\Delta \mathcal{B}^{(i)}\right)\tau\left(\overline{\mu}_{\otimes X_{1}}\circ\dots\circ \overline{\mu}_{\otimes X_{q}} \right),
\label{tr_2}
\end{align}
where $\mathcal{A} = \sum_{j=1}^{N_{H}} \mathcal{A}^{T}_{j} \mathcal{A}_{j}$ and $\Delta \mathcal{B}^{(i)} = \sum_{j=1}^{N_{H}} {\Delta \mathcal{B}^{(i)}_{j}}^{T} \Delta \mathcal{B}^{(i)}_{j}$. The last equality derives directly from assumption \ref{assump2}. Now we introduce another assumption for further analysis of MIPG.
\begin{lema}\label{lema}
	Two high dimensional square matrices (e.g. $A$, $B$) whose elements are generated independently from two random variables fulfill the following property 
	\begin{align}
		\tau (A \circ B) \approx \tau(A) \tau(B).
		\label{lema1}
	\end{align}
\end{lema}
\begin{proof}
	Firstly, the elements of $A$ and $B$ can be viewed as realizations of two underlying random variables; we denote them by $X_{A}$ and $X_{B}$, respectively. The left hand side of Equation \ref{lema1} becomes
	\begin{align}
		\tau (A \circ B) = \frac{1}{l_{A}} \tr(A \circ B) = \frac{1}{l_{A}} \sum_{j=1}^{l_{A}}A_{jj}B_{jj} \approx \mathbb{E}\left[ X_{A}X_{B} \right],
	\end{align}
	where $l_{A}$ is the size of $A$ and $A_{jj}$ denotes $A$'s element on $j$th row and $j$th column. Similarly, we have $B_{jj}$. Then the right hand side of equation \ref{lema1} becomes
	\begin{align}
		\tau(A) \tau(B) = \left( \frac{1}{l_{A}} \sum_{j=1}^{l_{A}}A_{jj} \right) \left( \frac{1}{l_{B}} \sum_{j=1}^{l_{B}}B_{jj} \right) \approx \mathbb{E}\left[ X_{A} \right] \mathbb{E}\left[ X_{B} \right],
	\end{align}
	where $l_{A}$ is the size of $A$ and $l_{A} = l_{B}$. Finally, by adopting the independence between $X_{A}$ and $X_{B}$, we complete the proof.
\end{proof}
Based on Lemma \ref{lema}, we make following assumption for MIPG.
\begin{assp}\label{assump3}
	We assume that the elements of tensor mean embedding of the variation of density of each parent node (e.g. $\Delta \mu_{\otimes X_{k}}^{(i)}$k = 1, \dots, q) of certain variable and that of the base of densities of other parent nodes (e.g. $\overline{\mu}_{\otimes X_{l}}, l \neq k$) are generated independently.
\end{assp}
A basic example implied by assumption \ref{assump3} is $\tau (\Delta \mu_{\otimes X_{k}}^{(i)} \circ \overline{\mu}_{\otimes X_{l}}) = \tau(\Delta \mu_{\otimes X_{k}}^{(i)}) \tau(\overline{\mu}_{\otimes X_{l}})$. $\Delta \mu_{\otimes X_{k}}^{(i)}$ depends only on $\Delta p^{(i)}(X_{k})$ and $\overline{\mu}_{\otimes X_{l}}$ depends only on $\overline{p}(X_{l})$. Based on the mutual independence among parent nodes of variables in MIPG, assumption \ref{assump3} further states that the tensor mean embedding of the variation of the density of a parent node is independent of that of the base of the density of another parent node. This can be easily extended to cases with more than two terms provided that the independence holds. Under assumption \ref{assump3}, equation \ref{tr_2} becomes
\begin{align}
\tau\left(\Delta\mu^{(i)}_{\otimes Y}\right)\approx &~\tau\left(\mathcal{A}\right)\tau\left(\Delta \mu_{\otimes X_{1}}^{(i)}\right)\tau\left(\mu_{\otimes X_{2}}\circ\dots\circ \mu_{\otimes X_{q}} \right)+\dots \nonumber \\
&\dots +\tau\left(\mathcal{A}\right)\tau\left(\mu_{\otimes X_{1}}\circ\dots\circ \mu_{\otimes X_{q-1}}\right)\tau\left(\Delta\mu_{\otimes X_{q}}^{(i)} \right) \nonumber \\
&+\tau\left(\Delta \mathcal{B}^{(i)}\right)\tau\left(\mu_{\otimes X_{1}}\circ\dots\circ \mu_{\otimes X_{q}} \right).
\label{mtp_1}
\end{align}
We introduce the following notations for simplicity:

\noindent
\textbf{Notation 2.} We denote the $k$th parent node of $Y$ by $pa_{k}(Y)$, $\tau\left(\Delta \mu_{\otimes Y}^{(i)}\right)$ by $\tau_{ y }^{(i)}$, $\tau\left(\Delta \mu_{\otimes X_{k}}^{(i)}\right)$ by $\tau_{ pa_k(y) }^{(i)}$, $\tau\left(\mathcal{A}\right)\tau\left(\mu_{\otimes X_{1}}\circ\dots\mu_{\otimes X_{k-1}} \circ \mu_{\otimes X_{k+1}} \circ \dots\circ \mu_{\otimes X_{q}} \right)$ by $c_{y|pa_{k}(y)}$ and $\tau\left(\Delta \mathcal{B}^{(i)}\right)\tau\left(\mu_{\otimes X_{1}}\circ\dots\circ \mu_{\otimes X_{q}} \right)$ by $\epsilon_{pa(y)\rightarrow y}^{(i)}$. We view each $\tau_{ y }^{(i)}$ as a realization of a random variable $\tau_{ y }$. Similarly, there are variables $\tau_{ pa_k(y) }, k=1, \dots, q$ and $\epsilon_{pa(y)\rightarrow y}$.

Then equation \ref{mtp_1} is formalized in the following proposition:
\begin{prop}\label{prop3}
	In an MIPG $\mathcal{G}$ of $p$ nodes where each variable $X_{m}$ has $q_m$ independent parent nodes,  if assumption \ref{assump1} to \ref{assump3} hold, the normalized traces of the tensor mean embedding of the variation of densities of all variables $\left( \tau_{x_{m}}, m = 1, \dots, p \right)$ fulfill a linear nongaussian acyclic model (LiNGAM) \citep{shimizu-et-al:lingam},
	\begin{align}
	\bm{\tau}_{x}&\approx \bm{C} \bm{\tau}_{x} + \bm{\epsilon},
	\end{align}
	where $\bm{\tau}_{x} = \left[\tau_{x_{1}}, \dots, \tau_{x_{p}} \right]^{T}$, coefficient matrix $\bm{C}$ whose element on $n$th row and $m$th column equals to $c_{x_{n}|x_{m}}$ could be permuted to a lower triangular matrix and $\bm{\epsilon} = \left[\epsilon_{pa(X_{1})\rightarrow X_{1}}, \dots, \epsilon_{pa(X_{p})\rightarrow X_{p}} \right]^{T}$.
\end{prop}
\begin{proof}
	First, $\tau_{x_{m}}$, where $m = 1, \dots, p$, could be arranged in an \emph{causal order} due to the acyclicity of the graph. Second, the noise term $\epsilon_{pa(x_{m})\rightarrow x_{m}}$, where $m = 1, \dots, p$, follows nongaussian distributions as shown in proposition \ref{prop1}. Thirdly, assumption \ref{assump2} ensures that $\tau_{x_{m}} \perp\!\!\!\perp \epsilon_{pa(x_{m})\rightarrow x_{m}}$ for $m = 1, \dots, p$.
\end{proof}

\begin{algorithm}[!ht]
	\caption{ENCI for causal graphs}
	\label{alg2}
	\renewcommand{\algorithmicrequire}{\textbf{Input:}}
	\renewcommand{\algorithmicensure}{\textbf{Output:}}
	\begin{algorithmic}[1]
		\REQUIRE $N$ data groups $\mathbf{G} = \{G_{1}, G_{2}, \dots, G_{N}\}$
		\ENSURE The estimated coefficient matrix $C_{ENCI}$ of the causal graph
		\STATE Normalize $X_{m}$ in each group for $m = 1, \dots, p$;
		\STATE Compute $\tau_{x_{1}}^{(i)}, \dots, \tau_{x_{p}}^{(i)}$ for $i=1, \dots, N$;
		\STATE Apply LiNGAM on $\tau_{x_{1}}, \dots, \tau_{x_{p}}$ and obtain the coefficient matrix $\mathbf{C}$;
		\STATE Denote the number of rows and columns with only one non-zero element by $n_{row}$ and $n_{col}$, respectively;
		\IF{$n_{row} > n_{col}$}
		\STATE Set elements in the rows with more than one non-zero element to be zero except for the maximal element and return the resulting matrix $C_{ENCI}$.
		\ELSIF{$n_{row} < n_{col}$}
		\STATE Set elements in the columns with more than one non-zero element to be zero except for the maximal element and return the resulting matrix $C_{ENCI}$.
		\ELSE
		\STATE Return $C$ as $C_{ENCI}$.
		\ENDIF
	\end{algorithmic}
\end{algorithm}
According to proposition \ref{prop3}, we can apply LiNGAM on the normalized traces of the tensor mean embedding of the variation of densities of all variables to infer the causal structure. However, the coefficient matrix $\mathbf{C}$ returned by LiNGAM needs to be further adjusted since LiNGAM is not restricted to the two kinds of causal graphs we are considering in this letter. Obviously for TSGs, each row of $\mathbf{C}$ contains at most one non-zero element. For MIPGs, each column contains at most one non-zero element since it can be obtained by reversing all directed edges of TSGs. Therefore, we first determine whether the returned coefficient matrix is more likely to be a TSG or MIPG by simply comparing the number of rows and columns with one non-zero element. Then we adjust those rows (columns) that violate the corresponding graph structure. The algorithm of extending ENCI to discover the causal structure of a graph with multiple variables are given in algorithm~\ref{alg2}.

\section{Experiment}
We conduct experiments on both synthetic and real data to verify the effectiveness of our proposed causal discovery algorithm. Unless specified, we adopt gaussian kernel with median $\left( d_{M} \right)$ as its kernel width across all subsections. The implementations of ENCI for cause-effect pairs\footnote{\url{https://github.com/amber0309/ENCI_cause-effect-pair}} and causal graphs\footnote{\url{https://github.com/amber0309/ENCI_causal-graph}} are available online. 

\subsection{Synthetic Cause-effect Pairs}
We generate the cause $X$ from the following family of distributions
\begin{align}
X\sim\frac { c_{1} }{ \sqrt { 2\pi { \left( 0.3 \right)  }^{ 2 } }  } { e }^{ -\frac { { \left( X-1 \right)  }^{ 2 } }{ 2{ \left( 0.3 \right)  }^{ 2 } }  } + \frac { c_{2} }{ \sqrt { 2\pi { \left( 0.3 \right)  }^{ 2 } }  } { e }^{ -\frac { { \left( X \right)  }^{ 2 } }{ 2{ \left( 0.3 \right)  }^{ 2 } }  }
+ \frac { c_{3} }{ \sqrt { 2\pi { \left( 0.3 \right)  }^{ 2 } }  } { e }^{ -\frac { { \left( X+1 \right)  }^{ 2 } }{ 2{ \left( 0.3 \right)  }^{ 2 } }  }, \nonumber
\end{align}
where $c_{1}$, $c_{2}$ and $c_{3}$ are randomly sampled from a uniformly distributed simplex. When generating a group of data, $c_{1}$ to $c_{3}$ are firstly sampled to determine the distribution of $X$. Then $40\sim50$ data points are sampled from the corresponding distribution to form a group, and 200 groups are generated in each experiment. The generating mechanism of $c_{1}$ to $c_{3}$ leads to the independence and difference of distributions in different groups. We conduct experiments with both an additive mechanism, $Y = f(X) + E$, and a multiplicative mechanism, $Y = f(X) \times E$. $E$ is the standard Gaussian noise. The function mapping $X$ to $Y$ of each group is randomly chosen from $f_{1}$ to $f_{7}$,
\begin{table}[H]
	\centering
	\begin{tabular}{llll}
		$f_{1}(x)=\frac{1}{x^{2} + 1}$ & $f_{2}(x)=sign(cx) \times (cx)^{2}$ & $f_{3}(x)=\cos(cxn)$ & $f_{4}(x)=x^{2}$ \\
		$f_{5}(x)=\sin(cx)$ &  $f_{6}(x)=2\sin(x) + 2\cos(x)$ & $f_{7}(x)=4\sqrt{|x|}$ & \\
	\end{tabular}
\end{table}
\noindent
where $c$ is a random coefficient independently and uniformly sampled from interval $[0.8, 1.2]$. Overall, $p(X)$ and function $f$ are fixed within each group, whereas they vary in different groups. 

We compare ENCI with ANM, PNL, IGCI and ECBP. These existing methods are applied in two different causal inference schemes: $(1)$ on the entire dataset, which is obtained by combing all groups (ALL) and $(2)$ on each group and choose their majority estimation to be their final causal direction estimation (MV). The experimental results of each setting are shown in Table~\ref{synpair}. Note that the accuracies of ECBP are from 50 independent experiments due to its high time complexity and that of other methods are from 100 independent experiments. 

	\begin{table*}[!ht]
		\centering
		\caption{Accuracy of synthetic cause-effect pairs}
		\label{synpair}
		\begin{tabular}{c|ccccccccc}
			\toprule
			\multirow{2}{*}{Mechanism} & \multirow{2}{*}{ENCI} & \multicolumn{2}{c}{ANM} & \multicolumn{2}{c}{PNL} & \multicolumn{2}{c}{IGCI} & \multicolumn{2}{c}{ECBP}\\
			& & MV & ALL & MV & ALL & MV & ALL & MV & ALL\\
			\midrule
			Additive & \textbf{100} & \textbf{100} & 63 & 99 & 50 & \textbf{100} & 66 & \textbf{100} & \textbf{100} \\
			Multiplicative & \textbf{100} & 0 & 26 & 4 & 5 & \textbf{100} & 90 & \textbf{100} & 88\\
			\bottomrule
		\end{tabular}
	\end{table*}

From the experimental results, we can see that ENCI, IGCI-MV, and ECBP-MV performs best compared with other cases. ANM and PNL could not make correct decision in both mechanisms at the same time, and the accuracy of IGCI-ALL is much lower than IGCI-MV, which is probably because of the influence of nonstationarity. ECBP takes non-stationarity into consideration so it achieves satisfactory accuracy in ECBP-ALL. However, we observe that its performance on multiplicative mechanism is not as good as ENCI in our experimental setting.

\subsection{Synthetic Causal Graph}
In this section, we show our experimental results of both kinds of causal graphs.

In the case of tree-structured graph, we conduct experiments on randomly generated graphs with 10 and 50 variables, respectively. First, the distributions of the root node is determined in the same way as the cause $X$ in the previous section. Then each effect is determined by a multiplicative mechanism from its parent node. The function $f$ is randomly chosen from $f_{1}$ to $f_{7}$, and all noise terms follow uniform distribution $\mathcal{U}(0, 1)$. Each time, 1000 groups of data are generated in total. Note that samples within each group are generated from a fixed causal model, but the distribution of the nodes and the mappings between them can vary in different groups. 

We compare ENCI with seven existing methods. ECBP \citep{zhang2015discovery}, ICA-LiNGAM \citep{shimizu-et-al:lingam}, DirectLiNGAM \citep{shimizu2011directlingam} and pairwiseLiNGAM \citep{hyvarinen2013pairwise} are directly applied after combining all groups of data. ANM \citep{hoyer2009nonlinear}, PNL \citep{zhang2009identifiability} and IGCI \citep{janzing2012information} are applied on each pair of adjacent nodes so we only have the proportion of correctly estimated edges (recall) for these three methods. Figure~3 shows one of the estimated results of the methods which are able to recover the causal structure. In each experiment, we compute the recall (and precision) of edge from the estimation results. The mean precision (prc) and recall (rcl) are given in column TSG of Table~\ref{synnet}\footnote{Note that the precision and recall of ECBP are computed from the skeleton instead of the directed graph.}. The results of ECBP on TSG with 10 and 50 variables are the mean of 50 and 20 independent experiments, respectively, due to its high time complexity. The results of other methods are the mean of 100 independent runs.

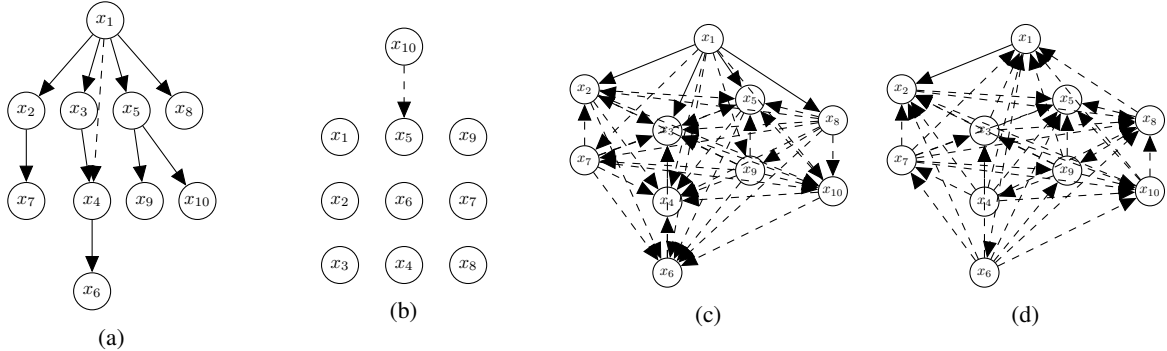
\begin{figure*}[!ht]
	\begin{subfigure}{.23\textwidth}
		\centering
		\begin{tikzpicture}[scale=0.7,transform shape,
		state/.style={circle,draw,thick,loop above,inner sep=0,minimum width=10}]
		\node[latent] (x1) {$x_{1}$} ; %
		\node[latent, below=of x1, xshift=-1.5cm] (x2) {$x_{2}$} ; %
		\node[latent, below=of x1, xshift=-0.5cm] (x3) {$x_{3}$} ; %
		\node[latent, below=of x1, xshift=0.5cm] (x5) {$x_{5}$} ; %
		\node[latent, below=of x1, xshift=1.5cm] (x8) {$x_{8}$} ; %
		\edge {x1}{x2, x3, x5, x8} ; %
		
		\node[latent, below=of x2] (x7) {$x_{7}$} ; %
		\node[latent, below=of x3, xshift=0.25cm] (x4) {$x_{4}$} ; %
		\node[latent, below=of x5, xshift=0.25cm] (x9) {$x_{9}$} ; %
		\node[latent, below=of x8, xshift=0.25cm] (x10) {$x_{10}$} ; %
		\edge {x2}{x7} ; %
		\edge {x3}{x4} ; %
		\edge {x5}{x9, x10} ; %
		\edge[style=dashed] {x1}{x4};
		
		\node[latent, below=of x4] (x6) {$x_{6}$} ; %
		\edge {x4}{x6} ; %
		\end{tikzpicture}
		\caption{}
		\label{Fig:syn_enci}
	\end{subfigure}
	\begin{subfigure}{.23\textwidth}
		\centering
		\begin{tikzpicture}[scale=0.7,transform shape,
		state/.style={circle,draw,thick,loop above,inner sep=0,minimum width=10}]
		\node[latent] (x1) {$x_{1}$} ; %
		\node[latent, below=of x1, yshift=0.5cm] (x2) {$x_{2}$} ; %
		\node[latent, below=of x2, yshift=0.5cm] (x3) {$x_{3}$} ; %
		\node[latent, right=of x3, xshift=-0.5cm] (x4) {$x_{4}$} ; %
		\node[latent, right=of x4, xshift=-0.5cm] (x8) {$x_{8}$} ; %
		\node[latent, right=of x2, xshift=-0.5cm] (x6) {$x_{6}$} ; %
		\node[latent, right=of x6, xshift=-0.5cm] (x7) {$x_{7}$} ; %
		\node[latent, right=of x1, xshift=-0.5cm] (x5) {$x_{5}$} ; %
		\node[latent, right=of x5, xshift=-0.5cm] (x9) {$x_{9}$} ; %
		\node[latent, above=of x5] (x10) {$x_{10}$} ; %
		\edge[style=dashed] {x10}{x5} ; %
		\end{tikzpicture}
		\caption{}
		\label{Fig:syn_ling}
	\end{subfigure}
	\begin{subfigure}{.25\textwidth}
		\centering
		\begin{tikzpicture}[scale=0.55,transform shape,
		state/.style={circle,draw,thick,loop above,inner sep=0,minimum width=10}]
		\node[latent] (x1) {$x_{1}$} ; %
		\node[latent, below=of x1, xshift=-3cm, yshift=0.5cm] (x2) {$x_{2}$} ; %
		\node[latent, below=of x1, xshift=-1cm, yshift=-0.5cm] (x3) {$x_{3}$} ; %
		\node[latent, below=of x1, xshift=1cm, yshift=0.25cm] (x5) {$x_{5}$} ; %
		\node[latent, below=of x1, xshift=3cm, yshift=-0.25cm] (x8) {$x_{8}$} ; %
		
		\node[latent, below=of x2] (x7) {$x_{7}$} ; %
		\node[latent, below=of x3] (x4) {$x_{4}$} ; %
		\node[latent, below=of x5] (x9) {$x_{9}$} ; %
		\node[latent, below=of x8] (x10) {$x_{10}$} ; %
		
		\node[latent, below=of x4] (x6) {$x_{6}$} ; %
		
		\edge{x1}{x2, x3, x5, x8};
		\edge[style=dashed]{x1}{x4, x6, x7, x9, x10};
		
		\edge[style=dashed]{x2}{x6, x4, x10, x5, x3};
		
		\edge[style=dashed]{x4}{x3};
		
		\edge[style=dashed]{x5}{x6, x4, x3};
		
		\edge[style=dashed]{x6}{x4, x3};
		
		\edge[style=dashed]{x7}{x4, x3, x6, x5, x10, x2};
		
		\edge[style=dashed]{x8}{x2, x3, x4, x5, x6, x7, x9, x10};
		
		\edge[style=dashed]{x9}{x2, x3, x4, x5, x6, x7, x10};
		
		\edge[style=dashed]{x10}{x3, x4, x5, x6};
		\end{tikzpicture}
		\caption{}
		\label{Fig:dling}
	\end{subfigure}
	\begin{subfigure}{.25\textwidth}
		\centering
		\begin{tikzpicture}[scale=0.55,transform shape,
		state/.style={circle,draw,thick,loop above,inner sep=0,minimum width=10}]
		\node[latent] (x1) {$x_{1}$} ; %
		\node[latent, below=of x1, xshift=-3cm, yshift=0.5cm] (x2) {$x_{2}$} ; %
		\node[latent, below=of x1, xshift=-1cm, yshift=-0.5cm] (x3) {$x_{3}$} ; %
		\node[latent, below=of x1, xshift=1cm, yshift=0.25cm] (x5) {$x_{5}$} ; %
		\node[latent, below=of x1, xshift=3cm, yshift=-0.25cm] (x8) {$x_{8}$} ; %
		
		\node[latent, below=of x2] (x7) {$x_{7}$} ; %
		\node[latent, below=of x3] (x4) {$x_{4}$} ; %
		\node[latent, below=of x5] (x9) {$x_{9}$} ; %
		\node[latent, below=of x8] (x10) {$x_{10}$} ; %
		
		\node[latent, below=of x4] (x6) {$x_{6}$} ; %
		
		\edge{x1}{x2};
		\edge[style=dashed]{x1}{x6};
		
		\edge[style=dashed]{x2}{x5, x8};
		\edge[style=dashed]{x3}{x1, x2, x8, x10, x5}
		\edge[style=dashed]{x4}{x1, x2, x8, x7, x10, x5, x3};
		\edge[style=dashed]{x6}{x2, x3, x8, x10, x9, x7, x5};
		\edge[style=dashed]{x7}{x8, x2, x9, x1, x5, x10, x3};
		\edge[style=dashed]{x8}{x1, x5, x9};
		\edge[style=dashed]{x9}{x2, x1, x5, x3, x4, x10};
		\edge[style=dashed]{x10}{x8, x2, x5, x1};
		\end{tikzpicture}
		\caption{}
		\label{Fig:pling}
	\end{subfigure}
	\caption{Examples of estimated results of (a) ENCI (b) ICA-LiNGAM (c) DirectLiNGAM (d) pairwiseLiNGAM.}
\end{figure*}

	\begin{table*}[!ht]
		\centering
		\caption{Accuracy of synthetic cause-effect pairs}
		\label{synnet}
		\begin{tabular}{c|cccccc}
			\toprule
			\multirow{3}{*}{Methods} & \multicolumn{4}{c}{TSG} & \multicolumn{2}{c}{MIPG} \\
			& \multicolumn{2}{c}{10 vars} & \multicolumn{2}{c}{50 vars} & \multicolumn{2}{c}{6 vars} \\
			& prc & rcl & prc & rcl & prc & rcl \\
			\midrule
			ENCI & \textbf{74.55} & 91.56 & \textbf{61.36} & 89.31 & \textbf{57.17} & 96.60  \\
			ECBP & 47.23 & 39.18 & 47.69 & 41.12 & 35.92 & \textbf{98.00} \\
			ICA-LiNGAM & 7.41 & 0.82 & 5.76 & 0.49 & 30.60 & 91.60 \\
			pairwiseLiNGAM & 16.82 & 84.11 & 3.65 & 91.16 & 13.47 & 40.40 \\
			DirectLiNGAM & 7.16 & 35.78 & 0.92 & 23.10 & 0.27 & 0.80 \\
			ANM & - & 24.33 & - & 26.42 & - & 6.60 \\
			PNL & - & 22.44 & - & 17.76 & - & 13.20 \\
			IGCI & - & \textbf{99.33} & - & \textbf{92.43} & - & 97.33  \\
			\bottomrule
		\end{tabular}
	\end{table*}

Next we conduct experiments on graphs that allow each variable to have multiple independent parent nodes. The experimental settings are similar to tree-structured case except that we generate 2000 data groups instead of 1000 and the ground truth of the synthetic network structure is fixed to be the graph on the right hand side of Figure~\ref{fig:exmp}. The mean precision and recall are given in the MIPG column of Table~\ref{synnet}. Note again that the results of ECBP are the mean of 20 independent experiments and that of other methods are the mean of 100 independent experiments.

The experimental results show a clear advantage of ENCI over ECBP, ICA-LiNGAM, pairwiseLiNGAM and DirectLiNGAM in estimating nonstationary causal graph. In both cases, ENCI achieves the highest precisions which are far higher than that of other methods. The recall of ENCI are also much higher compared with ICA-LiNGAM and DirectLiNGAM. Although in some cases ECBP and pairwiseLiNGAM return higher recall, their small precisions indicate that they find a large number of spurious edges, which makes their estimations less reliable. Comparing the recall of ENCI with ANM, PNL and IGCI, we find that ENCI still outperforms ANM and PNL. IGCI always performs the best among these four methods. Note that ANM, PNL and IGCI are not able to estimate the network structure and the recall of ENCI is relatively close to that of IGCI.

\subsection{Real Cause-effect Pairs}
This section and the next present the experimental results on real cause-effect pairs and causal graph, respectively. Note that experiments of applying ENCI on both real cause-effect pairs and real causal graphs are conducted on subsampled groups. In other words, we sampled data groups from the raw single data set to create the non-stationarity artificially and then applied ENCI on those randomly sampled groups to evaluate the performance of ENCI on real data.

We test the performance of ENCI on real world benchmark cause-effect pairs\footnote{https://webdav.tuebingen.mpg.de/cause-effect/.}. There are 106 pairs which come from 41 different data sets. Eight data sets are excluded in our experiment because they consists of either multivariate data or categorical data\footnote{Some of the existing methods or their implementations are not applicable to these data}. The corresponding pairs are of ID 47, 52, 53, 54, 55, 70, 71, 101 and 105. ENCI are compared with ANM \citep{hoyer2009nonlinear}, PNL \citep{zhang2009identifiability}, IGCI \citep{janzing2012information} and ECBP \citep{zhang2015discovery}. 

\begin{figure}[!hbtp]
	\begin{center}
		\includegraphics[width = 0.8\linewidth]{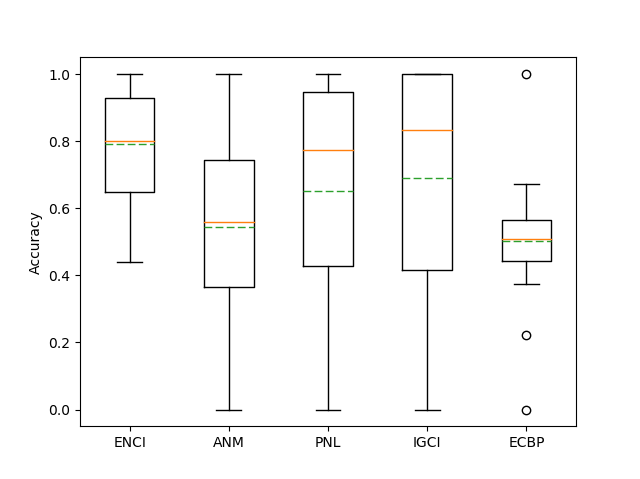}
	\end{center}
	\caption{Accuracy of methods on real world cause-effect pairs.}
	\label{real_pairs}
\end{figure}

We repeat 100 independent experiments for each pair and compute the percentage of correct inference. Then we compute the average percentage of pairs from the same source as the accuracy of the corresponding data set. In the experiment of ENCI, we apply ENCI on 90 groups, each of which consists of 50 to 60 points randomly sampled from the raw data without replacement. Four methods are directly applied on 90 points randomly sample from raw data without replacement in each experiment. Note that ENCI and IGCI is applied using different configurations, and the best result of each pair is adopted for evaluation. For ENCI, we test kernel width $d \in \{1/10d_{M}, 1/5d_{M}, 1/4d_{M}, 1/3d_{M}, 1/2d_{M}, d_{M}, 2d_{M}, 3d_{M}, 4d_{M}, 5d_{M}, 10d_{M} \}$, where $d_{M}$ is the median distance. For IGCI, we test different reference measures (i.e. uniform and gaussian) and estimators (i.e. entropy and integral estimation).

The summary of accuracies on 33 data sets of each method is given in Figure \ref{real_pairs} with orange solid line indicating median of accuracies and green dashed line indicating mean of accuracies. It shows that the performance of ENCI is satisfactory, with both median and mean accuracy about 79\%. IGCI also performs quite well, especially in terms of median, followed by PNL. ANM and ECBP performs poorly on these real cause-effect pairs, which might be due to their model restrictions. ENCI is much more stable than IGCI although its median accuracy is slightly lower. The results on real cause-effect pairs also indicate that ENCI could achieve satisfactory accuracy when applied on subgroups sampled from original data which does not strictly follow our non-stationary model.

\subsection{Real Causal Graph}
In this section, we test ENCI on a sociological data set from a data repository, General Social Survey\footnote{http://www.norc.org/GSS+Website/}.

This dataset consists of 6 observed variables, $x_{1}$: father's occupation level, $x_{2}$: son's income, $x_{3}$: father's education, $x_{4}$: son's occupation level, $x_{5}$: son's education, $x_{6}$: and number of siblings. We use the status attainment model based on domain knowledge~\citep{duncan1972socioeconomic} as the ground truth (see Figure~\ref{Fig:housing_gt}) and compare ENCI with ICA-LiNGAM, DirectLiNGAM and ECBP.
	\begin{figure}{}
		\centering
		\begin{tikzpicture}[scale=0.85,transform shape,
		state/.style={circle,draw,thick,loop above,inner sep=0,minimum width=10}]
		\node[latent] (x1) {$x_{1}$} ; %
		\node[latent, above=of x1, xshift=-1.5cm, yshift=-1cm] (x6) {$x_{6}$} ; %
		\node[latent, above=of x1, xshift=1.5cm, yshift=-1cm] (x3) {$x_{3}$} ; %
		\node[latent, below=of x1, xshift=-0.75cm] (x4) {$x_{4}$} ; %
		\node[latent, below=of x1, xshift=0.75cm] (x5) {$x_{5}$} ; %
		
		\node[latent, below=of x5, xshift=-0.75cm] (x2) {$x_{2}$} ; %
		
		\edge[] {x3}{x6};
		\edge[] {x6}{x3};
		\edge[] {x3}{x1};
		\edge[] {x1}{x3};
		\edge[] {x1}{x6};
		\edge[] {x6}{x1};
		
		\edge[] {x1, x3, x6}{x5} ;
		\edge[] {x1, x5, x6}{x4} ;
		\edge[] {x4, x5}{x2} ;
		
		\end{tikzpicture}
		\caption{Reference Graph of sociological dataset.}
		\label{Fig:housing_gt}
	\end{figure}
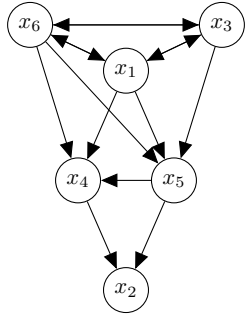

Before applying ENCI, we first adopt k-means++ \citep{arthur2007k} to cluster the original data into 15 clusters. In this way, we regard points within each cluster to be generated from the same causal model. Then we sample 1500 groups, which consists of 50 points sampled without replacement from each cluster with more than 50 points, and apply ENCI on these sampled groups. For ICA-LiNGAM, DirectLiNGAM and ECBP, we directly apply them on the original data set. We show one of the best results of ENCI (coefficient matrix $\mathbf{C}$ obtained from applying LiNGAM on $\tau_{x_{i}}, i=1, \dots, 6$) and the estimated graph from ICA-LiNGAM, DirectLiNGAM and ECBP in Figure~\ref{Fig:housing}.

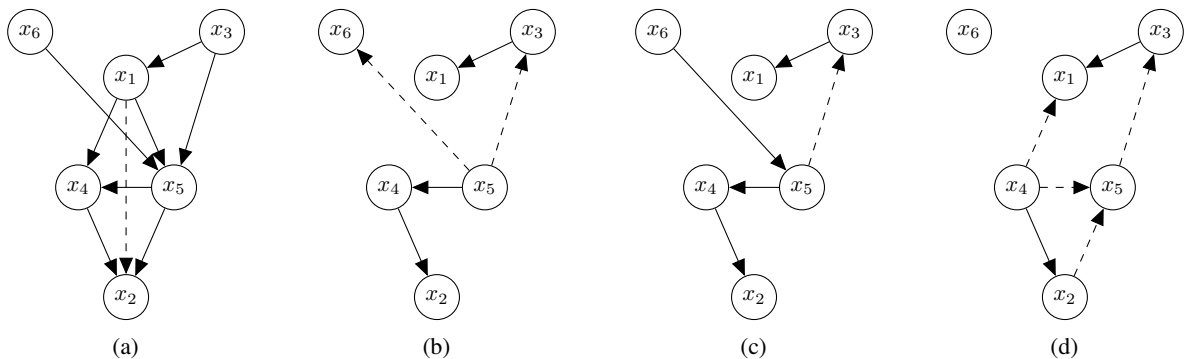
\begin{figure*}[!ht]
	\begin{subfigure}{.24\textwidth}
		\centering
		\begin{tikzpicture}[scale=0.85,transform shape,
		state/.style={circle,draw,thick,loop above,inner sep=0,minimum width=10}]
		\node[latent] (x1) {$x_{1}$} ; %
		\node[latent, above=of x1, xshift=-1.5cm, yshift=-1cm] (x6) {$x_{6}$} ; %
		\node[latent, above=of x1, xshift=1.5cm, yshift=-1cm] (x3) {$x_{3}$} ; %
		\node[latent, below=of x1, xshift=-0.75cm] (x4) {$x_{4}$} ; %
		\node[latent, below=of x1, xshift=0.75cm] (x5) {$x_{5}$} ; %
		
		\node[latent, below=of x5, xshift=-0.75cm] (x2) {$x_{2}$} ; %
		
		\edge[] {x3}{x1};
		\edge[] {x1, x3, x6}{x5} ;
		\edge[] {x1, x5}{x4} ;
		\edge[] {x4, x5}{x2} ;
		\edge[style=dashed] {x1}{x2}
		\end{tikzpicture}
		\caption{}
		\label{Fig:housing_enci}
	\end{subfigure}
	\begin{subfigure}{.25\textwidth}
		\centering
		\begin{tikzpicture}[scale=0.85,transform shape,
		state/.style={circle,draw,thick,loop above,inner sep=0,minimum width=10}]
		\node[latent] (x1) {$x_{1}$} ; %
		\node[latent, above=of x1, xshift=-1.5cm, yshift=-1cm] (x6) {$x_{6}$} ; %
		\node[latent, above=of x1, xshift=1.5cm, yshift=-1cm] (x3) {$x_{3}$} ; %
		\node[latent, below=of x1, xshift=-0.75cm] (x4) {$x_{4}$} ; %
		\node[latent, below=of x1, xshift=0.75cm] (x5) {$x_{5}$} ; %
		
		\node[latent, below=of x5, xshift=-0.75cm] (x2) {$x_{2}$} ; %
		
		\edge[] {x3}{x1} ;
		\edge[] {x5}{x4} ;
		\edge[] {x4}{x2} ;
		\edge[style=dashed] {x5}{x3, x6};
		\end{tikzpicture}
		\caption{}
		\label{Fig:housing_ling}
	\end{subfigure}
	\begin{subfigure}{.25\textwidth}
		\centering
		\begin{tikzpicture}[scale=0.85,transform shape,
		state/.style={circle,draw,thick,loop above,inner sep=0,minimum width=10}]
		\node[latent] (x1) {$x_{1}$} ; %
		\node[latent, above=of x1, xshift=-1.5cm, yshift=-1cm] (x6) {$x_{6}$} ; %
		\node[latent, above=of x1, xshift=1.5cm, yshift=-1cm] (x3) {$x_{3}$} ; %
		\node[latent, below=of x1, xshift=-0.75cm] (x4) {$x_{4}$} ; %
		\node[latent, below=of x1, xshift=0.75cm] (x5) {$x_{5}$} ; %
		
		\node[latent, below=of x5, xshift=-0.75cm] (x2) {$x_{2}$} ; %
		
		\edge[] {x3}{x1} ;
		\edge[] {x5}{x4} ;
		\edge[] {x4}{x2} ;
		\edge[] {x6}{x5} ;
		\edge[style=dashed] {x5}{x3};
		\end{tikzpicture}
		\caption{}
		\label{Fig:housing_dling}
	\end{subfigure}
	\begin{subfigure}{.24\textwidth}
		\centering
		\begin{tikzpicture}[scale=0.85,transform shape,
		state/.style={circle,draw,thick,loop above,inner sep=0,minimum width=10}]
		\node[latent] (x1) {$x_{1}$} ; %
		\node[latent, above=of x1, xshift=-1.5cm, yshift=-1cm] (x6) {$x_{6}$} ; %
		\node[latent, above=of x1, xshift=1.5cm, yshift=-1cm] (x3) {$x_{3}$} ; %
		\node[latent, below=of x1, xshift=-0.75cm] (x4) {$x_{4}$} ; %
		\node[latent, below=of x1, xshift=0.75cm] (x5) {$x_{5}$} ; %
		
		\node[latent, below=of x5, xshift=-0.75cm] (x2) {$x_{2}$} ; %
		
		\edge[] {x3}{x1};
		\edge[style=dashed] {x4}{x1};
		\edge[style=dashed] {x4, x2}{x5} ;
		\edge[style=dashed] {x5}{x3} ;
		\edge[] {x4}{x2}
		\end{tikzpicture}
		\caption{}
		\label{Fig:housing_ecbm}
	\end{subfigure}
	\caption{Estimated graph of (a) ENCI (b) ICA-LiNGAM (c) DirectLiNGAM (d) ECBP.}
	\label{Fig:housing}
\end{figure*}

ENCI outperforms ICA-LiNGAM, DirectLiNGAM which is consistent with our expectation since they are developed for linear stationary models. ENCI also outperforms ECBP which may be due to the lack of nonstationarity of the raw data. 

There are two facets of ENCI worth noting from the results of real data experiments of both pairs and causal graphs. First, ENCI is applied on subgroups sampled from the raw data since each set of real data is a single collection of observations and does not contain the form of nonstationarity our model assumes. However, the results of ENCI on real pairs is quiet competitive and it performs much better than LiNGAM family methods in real causal graph. This gives some evidence that our model could achieve satisfactory performance with subtle nonstationarity, which may be simply generated by subsampling a single data set. Second, the reference graph in the real graph experiment does not strictly fulfill the requirements of ENCI, but we obtain acceptable estimation results, which implies that ENCI may be applicable for other kinds of causal graphs.

\section*{Conclusion}
In this paper, we introduce the nonstationary causal model and prove the asymmetry of non-stationarity between the causal direction and anti-causal direction based on certain assumptions. By exploiting this asymmetry, we propose a reproducing kernel Hilbert space embedding-based method, ENCI, to infer the causal structure of both cause-effect pairs and two kinds of causal graphs. Theoretical analysis and experiments show the advantage of ENCI over existing methods based on fixed causal models when being applied on nonstationary passive observations. 

Compared with ECBP which is also for non-stationary causal model inference, the theoretical scope of application of ENCI is more restricted in the sense that we require non-stationarity in both $p(X)$ and $p(Y|X)$, whereas ECBP would also work when only one of them is nonstationary. In addition, ENCI requires nonstationarity exists in every variable of a causal graph, whereas ECBP only requires the existence of nonstationarity. However, ENCI outperforms ECBP on the experiments of both real cause-effect pairs and causal graph in which the data generating process does not strictly follow our model assumptions and the nonstationarity among artificial groups is subtle. Therefore, we deem that ENCI could be applied on a much wider scope of problems in reality and achieve satisfactory performance. In this way, ENCI is eligible to join a pool of state-of-the-art algorithms for learning general causal models.

\subsection*{Acknowledgments}
We would like to thank Biwei Huang and Kun Zhang for providing the code of Enhanced Constraint-based Procedure (ECBP).

\bibliographystyle{unsrt}
\bibliography{references}

\end{document}